%% file: iclr2027_conference.tex
\documentclass{article} 
\usepackage{iclr2027_conference,times}

\input{math_commands.tex}

\usepackage{graphicx, wrapfig, booktabs, multirow, subcaption, amsthm}
\usepackage[table]{xcolor}

\usepackage{hyperref}
\usepackage{url}

\newtheorem{lemma}{Lemma}
\newtheorem{proposition}{Proposition}

\title{Prior-Guided Single-Exposure HDR Novel View Synthesis via Base-Gain Representation}

\author{Armin Shafiee Sarvestani \qquad Sheyang Tang \qquad Jialu Xu \qquad Zhou Wang \\
Department of Electrical and Computer Engineering, University of Waterloo\\
\texttt{\{a5shafie, sheyang.tang, jialu.xu, zhou.wang\}@uwaterloo.ca} \\
}

\iclrfinalcopy 
\begin{document}

\maketitle
\lhead{arXiv Preprint}
\begin{abstract}
Single-exposure high dynamic range (HDR) novel view synthesis (NVS) recovers HDR scenes from low dynamic range (LDR) views captured at a single exposure and estimates camera response function (CRF) for LDR re-exposure. Without HDR ground-truth, using single-exposure LDR views as the sole source of direct supervision underconstrains HDR reconstruction and cannot ensure its faithful recovery, motivating the need for learned LDR-to-HDR reconstruction priors. Moreover, single-exposure LDR views supervise CRF only over a limited range, necessitating a learnable prior for reliable re-exposure beyond it. We propose PriorGS, a 3D Gaussian Splatting framework that addresses both challenges by (1) guiding HDR recovery with pseudo-HDR supervision from 2D single-exposure HDR reconstruction methods while anchoring it to input LDR views through a closed reconstruction loop; and (2) regularizing CRF modeling with a learnable prior that guarantees monotonicity and promotes plausible responses. Yet cross-view inconsistencies in independently generated pseudo-HDR targets impair scene reconstruction when HDR is encoded directly as Gaussian's color attribute. We therefore introduce a 3D base–gain HDR representation that factors HDR radiance into an LDR base and an overlaid 3D gain field, promoting robustness to such inconsistencies and enabling 3D HDR appearance editing. Extensive experiments demonstrate PriorGS's substantial superiority over baselines. 
\end{abstract}

\section{Introduction} \label{sec:intro}
High dynamic range (HDR) novel view synthesis (NVS) aims to reconstruct a 3D scene's geometry and HDR radiance from multi-view low dynamic range (LDR) observations, with the dual objective of synthesizing novel HDR views and rendering their LDR counterparts at unseen exposures. Compared with conventional LDR-NVS \citep{Nerf, GaussianSplats}, HDR-NVS represents a wider range of scene illumination, preserving textures in both highlights and shadows. This richer representation improves the visual appeal and realism of reconstructed scenes and benefits many downstream applications \citep{photorealistic_rendering, inverse_render, robotics}. Most HDR-NVS methods rely on multi-exposure LDR captures \citep{HDR_Nerf, HDRGS, GaussHDR} that typically require static scenes and tedious acquisition to avoid motion artifacts. Motivated by this limitation, recent methods, including ours, target HDR-NVS with more accessible single-exposure LDR inputs \citep{SeHDR, MonoHDR}. Also, as calibrated per-view HDR capture is rarely practical, we target HDR recovery without access to HDR ground-truth (GT).

HDR scene radiance represents linear scene-referred light over a wide intensity range. During LDR capture, this radiance is converted to sensor exposure, transformed by a nonlinear camera response function (CRF), quantized, and partially lost through highlight clipping and shadow noise \cite{EMoR}. Recovering HDR radiance from this degraded single-exposure observation is hence fundamentally underconstrained \citep{cheat_sihdr}, motivating 2D single-exposure HDR reconstruction methods that learn LDR-to-HDR reconstruction priors from large data \citep{intrinsic_hdr, single-hdr}. The same challenge arises in single-exposure HDR-NVS: using single-exposure views as the only source of direct supervision leaves HDR reconstruction underconstrained and cannot ensure its accurate recovery. This motivates the need for learned LDR-to-HDR reconstruction priors to guide HDR recovery in single-exposure HDR-NVS, a critical factor overlooked by prior works. Figure~\ref{fig:hdr_fail} illustrates this: when trained on the same single-exposure inputs, existing methods lacking learned LDR-to-HDR priors reproduce the observed exposure yet recover an incorrect HDR radiance. Our method, however, achieves high-fidelity HDR recovery by leveraging this prior.
\begin{figure}[t]
    \begin{center}
    \includegraphics[width=0.95\linewidth]{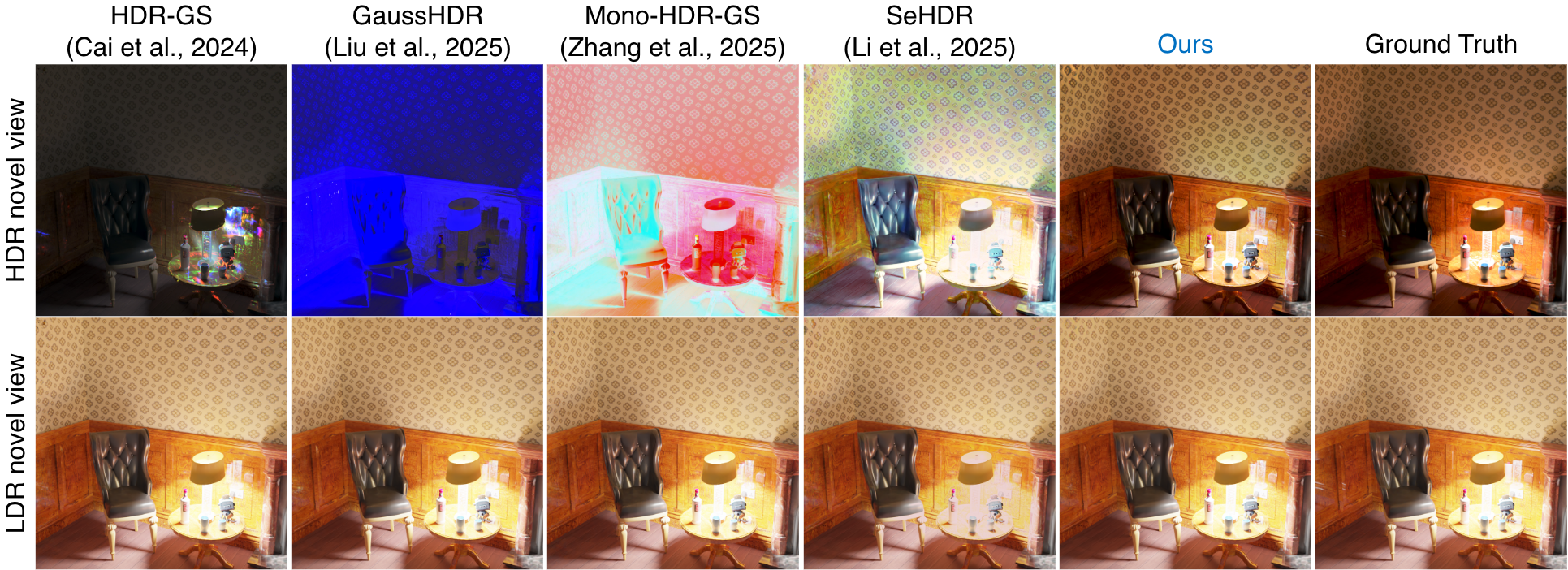}
    \end{center}
    \caption{Given the same single-exposure inputs, methods lacking learned LDR-to-HDR priors reconstruct the LDR novel view faithfully but fail at HDR radiance recovery. We leverage this prior to recover radiance much closer to the GT. HDR results are tone-mapped identically for visualization.}
    \label{fig:hdr_fail}
\end{figure}

Beyond high-fidelity HDR recovery, a holistic HDR-NVS method must accurately re-expose HDR radiance to LDR views at unseen exposures through CRF modeling, an objective often bypassed by existing single-exposure methods. Unlike multi-exposure captures that enable supervision for a wide CRF range, single-exposure views supervise only the camera’s limited unsaturated range at one exposure, leaving the CRF unconstrained elsewhere. Existing single-exposure HDR-NVS methods either (1) learn flexible MLP-based CRFs \citep{MonoHDR} that fit the observed exposure yet neither generalize well to unseen ones nor ensure the expected monotonicity (i.e. LDR intensity grows with input irradiance) \citep{EMoR}; or (2) fix a heuristically selected CRF per scene \citep{SeHDR} that risks response mismatch and precludes training-time adaptation. Single-exposure HDR-NVS thus needs a learnable CRF prior that regularizes response beyond the supervised range, guarantees monotonicity, and promotes plausible responses while retaining scene adaptability.
\begin{wrapfigure}{r}{0.45\textwidth}
    \begin{center}
    \includegraphics[width=\linewidth]{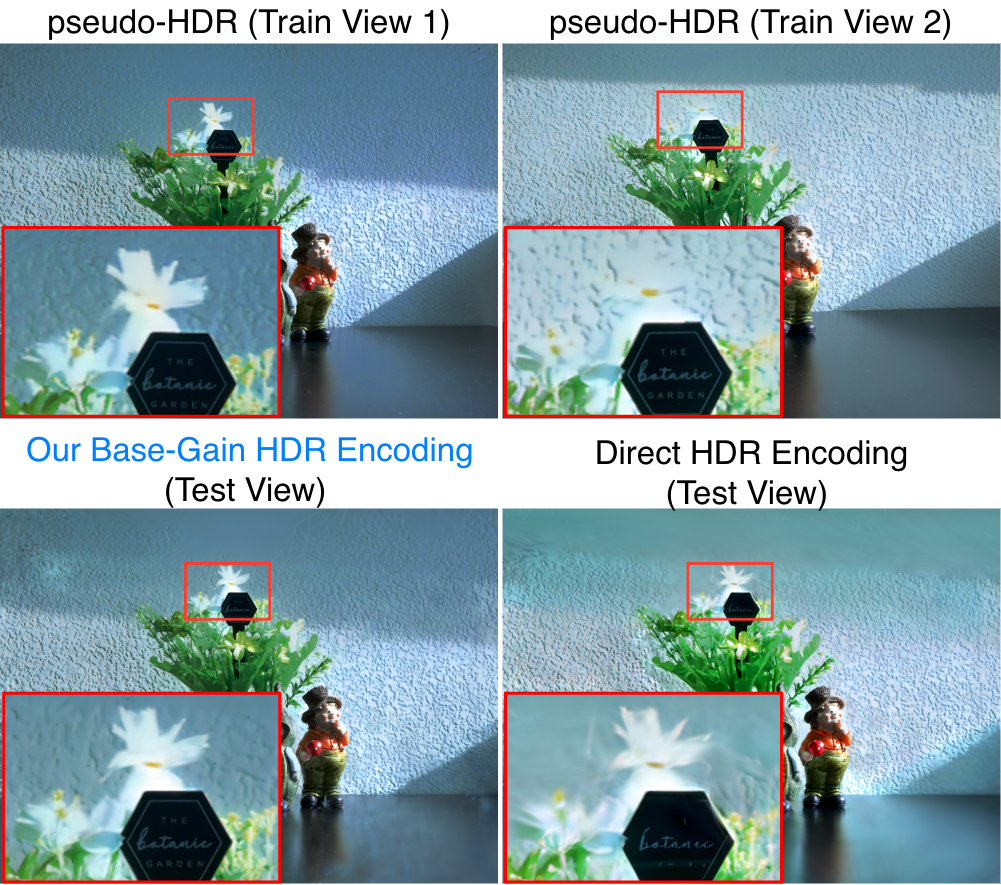}
    \end{center}
    \caption{(\textit{top row}) Cross-view inconsistencies in pseudo-HDR targets  (\textit{bottom row}) produce severe artifacts w/ direct HDR encoding, while our base–gain representation promotes robustness to such inconsistencies.}
    \label{fig:view_inconsist}
\end{wrapfigure}

In response, we propose PriorGS, a novel single-exposure HDR-NVS method based on 3D Gaussian Splatting (3DGS) \citep{GaussianSplats} that incorporates \textit{LDR-to-HDR reconstruction} and \textit{CRF priors} within a \textit{base–gain HDR representation}. To incorporate learned LDR-to-HDR reconstruction priors, we supervise HDR rendering at training views using pseudo-HDR targets generated by 2D single-exposure HDR reconstruction methods. However, as these targets are estimated separately per views, they may exhibit cross-view inconsistencies. As shown in Figures \ref{fig:view_inconsist} and \ref{fig:apx_abl_direct_hdr}, when HDR radiance is encoded directly as Gaussian's color attribute (direct HDR encoding), such inconsistencies impair scene reconstruction, causing severe color and geometry artifacts. To mitigate this, we represent HDR radiance as an LDR base supervised by LDR views and extended to HDR through a 3D gain field. This base-gain representation (1) anchors the scene to the more reliable LDR views while modeling the less reliable HDR extension by the 3D gain field, thereby promoting robustness to cross-view inconsistencies (Figures \ref{fig:view_inconsist} and \ref{fig:apx_abl_direct_hdr}); and (2) enables post-training HDR appearance editing in 3D without altering the LDR base. Next, we regularize CRF modeling with a learnable prior derived from a database of CRFs \citep{EMoR}. The CRF then re-exposes HDR renderings to match the same LDR views, closing an LDR-base $\rightarrow$ HDR $\rightarrow$ re-exposed LDR loop that anchors HDR recovery to accurate LDR reconstruction at both ends. 

Our contributions can be summarized as: (1) We propose PriorGS, a novel 3DGS-based single-exposure HDR-NVS framework that guides HDR reconstruction with learned LDR-to-HDR reconstruction priors while enforcing consistent LDR reconstruction before and after HDR recovery; (2) We introduce a base-gain 3D HDR representation that factors scene radiance into an LDR base and an overlaid 3D gain field, promoting robustness to cross-view pseudo-HDR inconsistencies and enabling HDR appearance editing in 3D; (3) We regularize CRF modeling in single-exposure HDR-NVS with a learnable prior, guaranteeing monotonicity and promoting plausible responses while retaining scene adaptability; and (4) Extensive experiments demonstrate that PriorGS substantially outperforms existing methods both quantitatively and qualitatively.

\section{Related Works}
\textbf{HDR Novel View Synthesis} seeks to reconstruct HDR scenes from either (1) LDR captures \citep{HDR_Nerf, ltm-nerf, HDRGS, NeXF}; (2) RAW images \citep{rawnerf, rawgs, 3dgs-in-the-dark, hdrsplat}; or (3) event–image pairs \citep{event1, event2}. We target the first as LDR images are easier to acquire and more widely available. Most LDR-based methods require multi-exposure LDR captures to jointly learn HDR radiance and LDR re-exposure \citep{HDRGS}, with recent approaches exploring unified 3D and 2D local tone-mapping \citep{GaussHDR}, physically-inspired appearance modeling \citep{PhysHDR}, and color-geometry-exposure disentanglement \citep{expo-gs}. However, these methods rely on multi-exposure supervision and thus do not address the underconstrained nature of single-exposure HDR-NVS. This limitation, together with tedious acquisition of multi-exposure images has motivated recent HDR-NVS methods using more accessible single-exposure LDR inputs \citep{SeHDR, MonoHDR}. Among single-exposure methods, SeHDR \citep{SeHDR} fuses bracketed 3D Gaussians into an HDR representation, whereas Mono-HDR-GS \citep{MonoHDR} learns MLP-based LDR-to-HDR and HDR-to-LDR mappings. However, without HDR GT, their HDR radiance recovery remains  underconstrained, lacking direct guidance for accurate HDR reconstruction and thus leading to unreliable recovery. PriorGS overcomes this by guiding HDR recovery with learned LDR-to-HDR reconstruction priors. Outside exposure-controllable HDR-NVS, PanoHDR-NeRF \citep{PanoHDR-NeRF} and GaSLight \citep{GasLight} use LDR-to-HDR predictions to train 3D HDR representations for relighting. However, they neither enforce consistency with original LDR views nor support LDR re-exposure. In contrast, PriorGS anchors prior-guided HDR recovery to accurate LDR reconstruction and re-exposure through closed-loop reconstruction. Finally, HDR-NVS requires reliable CRF modeling for re-exposure at unseen exposures. Multi-exposure methods employ either neural tone mappers \citep{HDRGS, GaussHDR, instanthdr} or explicit response parameterizations \citep{hdr-plenoxels}. These CRFs are constrained by multi-exposure inputs across a wide response range, while single-exposure inputs supervise only a narrow unsaturated region, leaving the remaining range unconstrained. Among single-exposure methods, Mono-HDR-GS \citep{MonoHDR} learns a flexible MLP-based CRF that fits the observed exposure but offers poor generalization beyond it and cannot ensure monotonicity. SeHDR \citep{SeHDR} instead heuristically fixes a CRF per scene but risks response mismatch and forfeits training-time adaptation. To address both limitations, PriorGS regularizes CRF modeling with a learnable prior that guarantees monotonicity expected of CRFs \citep{EMoR} and promotes plausible responses.

\textbf{Single-exposure 2D HDR Reconstruction} aims to recover an HDR image from a single LDR observation by reversing camera pipeline and inferring radiance degraded or lost through noise, quantization, and saturation. Approaches range from CNN- or Transformer-based prediction \citep{HDRCNN, ExpandNet, single-hdr, intrinsic_hdr, deephdr, Trasnformer-HDR} to GAN- and diffusion-based generation \citep{lediff, glowGAN, Expo-CM, x2hdr}, with recent works adopting dual-layer LDR-gain formulation that generates a 2D gain map over an LDR base \citep{gmnet, gmdiffusion}. A solution for single-exposure HDR-NVS is to use these methods to generate pseudo-HDR targets for views and train 3DGS with direct HDR encoding--an encoding prevalent in prior works \citep{HDRGS, GaussHDR}. However, cross-view inconsistencies in per-view generated pseudo-HDR targets impair scene reconstruction. PriorGS instead extends the 2D LDR–gain formulation to 3D, yielding a base–gain HDR representation that promotes robustness to such inconsistencies and enables 3D HDR appearance editing.

\section{Proposed Method} \label{sec:prop_method}
\begin{figure}[t]
    \begin{center}
    \includegraphics[width=0.95\linewidth]{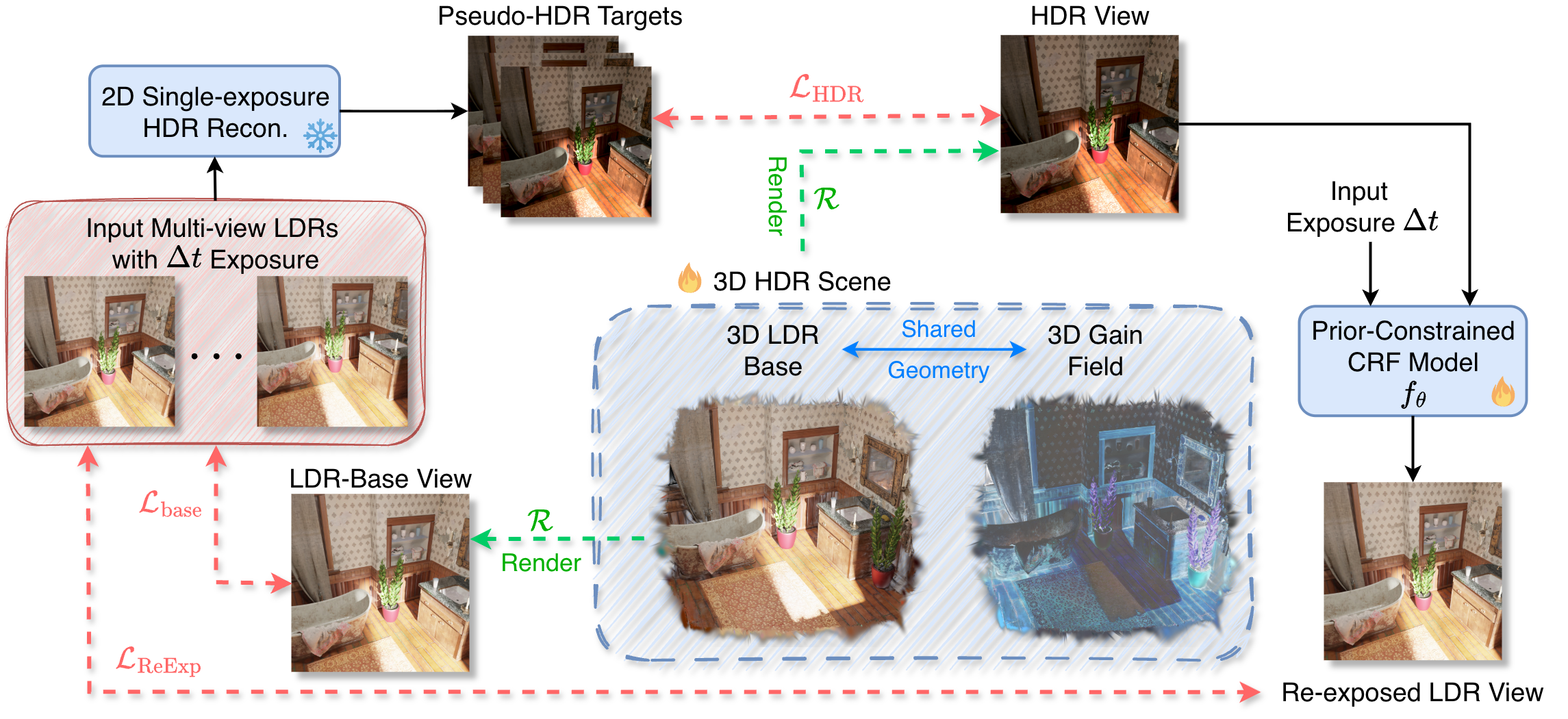}
    \end{center}
    \caption{Overview of \textbf{PriorGS}. Given single-exposure multi-view LDR captures of a scene, we reconstruct the 3D HDR scene as an LDR base extended by a 3D gain field. Input views supervise the LDR base, while pseudo-HDR targets generated by a 2D single-exposure HDR reconstruction model impose LDR-to-HDR prior on HDR recovery. We then re-expose the HDR view by a prior-constrained learnable CRF to match the same input view, which closes the reconstruction loop.}
    \label{fig:framework}
\end{figure}
\textbf{Problem Statement.} Given $V$ posed LDR views $L^v$ of a 3D scene, $\{ (L^v, \bm\pi^v)\}^V_{v=1}$, captured with camera parameters $\bm\pi^v$ at the same exposure $\Delta t$, we aim to jointly recover the HDR scene as a set of $N$ Gaussian primitives $\sG$ \citep{GaussianSplats} together with a scene-specific CRF $f_\vtheta$, where
\begin{equation}
    \sG = \{G_i(\vmu_i, \mSigma_i, \alpha_i, \bm\beta_i^c, \bm\beta_i^g)\}_{i=1}^{N}.
\end{equation}
Here, $G_i$ represents the $i$-th Gaussian and $\vtheta$ parameterizes learnable CRF $f_\vtheta$ for LDR re-exposure. Each Gaussian $G_i$ is composed of the following learnable parameters: (1) a position $\vmu_i \in \R^3$; (2) a 3D covariance $\mSigma_i = \mR_i\mS_i\mS_i^\mathrm{T}\mR_i^\mathrm{T}$ represented by rotation and scale matrices $\mR_i,\mS_i \in \R^{3\times3}$; (3) an opacity $\alpha_i \in \R$ controlling transparency; (4) RGB spherical harmonics (SH) coefficients $\bm{\beta}_i^c$ encoding view-dependent LDR-base color $\vc_i^v \in \R^3$; and (5) a second set of RGB SH coefficients $\bm{\beta}_i^g$  encoding view-dependent log-gain $\vg_i^v\in\R^3$ that lifts the LDR base to HDR radiance. We follow the standard 3DGS rasterization of \citet{GaussianSplats} and refer readers there for details.

\textbf{Overview.} Figure~\ref{fig:framework} shows an overview of PriorGS. We represent the 3D HDR scene using shared Gaussian geometry with two color components: an LDR base and a 3D gain field. Each Gaussian's gain lifts its LDR-base color to HDR radiance before rasterization, producing LDR-base and HDR views from the same geometry. A pretrained single-exposure HDR reconstruction method generates a pseudo-HDR target for each training view. The input LDR view and its pseudo-HDR target then supervise the LDR-base and HDR renderings, respectively. We next re-expose the HDR rendering with a prior-constrained learnable CRF and require the output to match the same LDR view. 

\subsection{Prior-Guided HDR Reconstruction via Base–Gain HDR Representation}
\textbf{Pseudo-HDR Guidance.} In Section \ref{sec:intro}, we discussed the necessity of using learned LDR-to-HDR reconstruction priors for accurate HDR reconstruction from single-exposure LDR views. We incorporate this prior through direct pseudo-HDR supervision: before training, we feed each training LDR view to a pretrained 2D single-exposure HDR reconstruction method and use the resulting pseudo-HDR target to supervise the rendered HDR view during training.

\textbf{Base-Gain HDR Representation.} As pseudo-HDR targets are generated independently for each view, they may have cross-view inconsistencies. When HDR radiance is encoded directly as a free-standing Gaussian color attribute (direct HDR encoding), such inconsistencies can impair scene reconstruction by driving unrestrained erroneous updates to Gaussians' attributes, resulting in severe color and geometry artifacts (see Figures \ref{fig:view_inconsist} and \ref{fig:apx_abl_direct_hdr}). To mitigate this and inspired by the modern LDR-gain 2D HDR formats \citep{ISOGainMap, UltraHDR, Gain-MLP}, we propose a base-gain 3D HDR representation that factors HDR radiance into an LDR base and an overlaid 3D gain field. Specifically, we implement each Gaussian $G_i$ to carry two sets of RGB SH
coefficients, $\bm{\beta}_i^c$ and $\bm{\beta}_i^g$, whose evaluations
for view $v$ produce its LDR-base color $\vc_i^v$ and log-gain
$\vg_i^v$, respectively. We then derive its view-dependent HDR radiance $\vh_i^v \in \R^3$ as 
\begin{equation} \label{eq:gain_eq}
    \vh_i^v = \left(\vl_i^v + \delta\right)\odot2^{\textstyle \,\vg_i^v}-\delta,
\end{equation}
where $\vl_i^v$ is the linear-light value obtained from $\vc_i^v$ via inverse gamma encoding, and $\delta=1/64$ facilitates gain application near blacks ($\vl_i^v=0$). Finally, for a camera $\bm\pi^v$, we render the LDR base $\widehat L_{\mathrm{B}}^v$ and HDR radiance $\widehat H^v$ with shared geometry but different color attributes $\vc_i^v$ and $\vh_i^v$:
\begin{equation}
    \widehat L_{\mathrm{B}}^v
    =
    \mathcal{R}\!\left(\sG,\bm\pi^v;\{\vc_i^v\}_{i=1}^{N}\right),
    \qquad
    \widehat H^v
    =
    \mathcal{R}\!\left(\sG,\bm\pi^v;\{\vh_i^v\}_{i=1}^{N}\right).
\end{equation}
This enables joint supervision of LDR-base and HDR renderings by LDR views and pseudo-HDR targets, respectively. Our base-gain representation thus models HDR radiance as a gain-based extension of LDR base rather than as a free color field. The reliable LDR views therefore anchor the base appearance and shared geometry, while 3D gain field supplies dedicated capacity for modeling the less reliable HDR extension, thereby promoting robustness to cross-view inconsistencies (see Figures \ref{fig:view_inconsist} and \ref{fig:apx_abl_direct_hdr}). Appendix \ref{subsec:apx_base_gain} provides further discussions.

\subsection{Prior-Constrained CRF Modeling \& Closed-loop Reconstruction} \label{subsec:crf_model}
\textbf{Prior-Constrained CRF Modeling}. Section \ref{sec:intro} motivated regularized CRF modeling in single-exposure HDR-NVS with a learnable prior. We derive this prior from DoRF \citep{EMoR}, a database of CRFs, by adopting the two-parameter monotonic approximation of its mean response proposed by \citet{HDRCNN}. As DoRF responses and the approximation are defined over normalized irradiance $x\in[0,1]$, we prepend a learnable soft normalization that preserves monotonicity and enhances representational capacity on DoRF responses (see Appendix~\ref{subsec:apx_crf_analysis} for proof and capacity analysis). Let $\vtheta=(\bm\phi,\bm\psi)$, $\bm\phi=(\sigma,n)$ and $\bm\psi=(a,b)$. Our CRF is defined as
\begin{equation} \label{eq:crf_model}
\begin{gathered}
    f_{\vtheta}(E)
    = t_{\bm\phi}\!\left(\eta_{\bm\psi}(E)\right), \\
    t_{\bm\phi}(x)
    = (1+\sigma)\frac{x^n}{x^n+\sigma},
    \qquad
    \eta_{\bm\psi}(E)
    = \operatorname{sigmoid}\!\left(a\ln(E+\epsilon)+b\right).
\end{gathered}
\end{equation}
Here $\sigma, n> 0$ parameterize the mean response approximation, while $a>0$ and $b \in \R$ govern the normalization. We learn all four parameters per scene to achieve scene-adaptability and share the learned CRF across RGB channels. Our CRF is strictly monotonic by construction, ensuring LDR intensity increases with irradiance, while DoRF-derived form of $t_{\bm\phi}$ promotes plausible responses.

\textbf{Closed-loop Reconstruction}. As shown in Figure~\ref{fig:framework}, we re-expose the HDR rendering $\widehat H^v$ at the training exposure $\Delta t$ using the learned CRF to obtain the re-exposed LDR view as
\begin{equation}
    \widehat L_{\mathrm{R}}^v(\Delta t)
    =
    f_{\vtheta}
    \left(
        \Delta t\,\widehat H^v
    \right).
\end{equation}
At test time, varying $\Delta t$ enables re-exposure at arbitrary exposures. By supervising the re-exposed view $\widehat L_{\mathrm{R}}^v(\Delta t)$ with the input LDR view, we close the reconstruction loop $\widehat L_{\mathrm{B}}^v \rightarrow \widehat H^v \rightarrow \widehat L_{\mathrm{R}}^v(\Delta t)$, where HDR reconstruction is anchored upstream to LDR base through 3D gain field and constrained downstream to recover the same LDR view after re-exposure. Consequently, HDR recovery is not only guided by LDR-to-HDR prior but also constrained to consistently reproduce input LDR views.


\textbf{Training Objective}. We jointly optimize reconstruction of LDR-base, HDR, and re-exposed LDR views. We follow prior works \citep{HDRGS, SeHDR} to use a common reconstruction loss that combines pixel-wise $\mathcal L_1$ loss with D-SSIM \citep{ssim} loss as
\begin{equation}
    D_\lambda(A, B) = \mathcal L_1(A, B) + \lambda\,\mathcal L_{\operatorname{D-SSIM}}(A, B),
\end{equation}

where $\lambda$ balances the two terms. The overall training objective is then formulated as
\begin{align} \label{eq:loss}
    \mathcal L 
    &= \mathcal L_{\mathrm{base}} + w_1\,\mathcal L_{\mathrm{HDR}} + w_2\,\mathcal L_{\mathrm{ReExp}}, \\
    &= D_\lambda(\widehat L_{\mathrm{B}}^v,L^v) + w_1\,D_\lambda\left(T_\mu(\widehat H^v),T_\mu(\widetilde H^v)\right) +  w_2\,D_\lambda\left(\widehat L_{\mathrm{R}}^v(\Delta t),L^v\right).
\end{align}
Here, $\widetilde H^v$ denotes the pseudo-HDR target, $w_1 = w_2 = 0.5$ control the effect of loss terms, and $T_\mu$ represents $\mu$-law range compressor \citep{kalantari2017} widely used in prior works for HDR benchmarking and loss optimization \citep{HDRGS, MonoHDR}. 

\section{Experiments}
\subsection{Experimental Settings} \label{sec:exp_settings}
\begin{table}[t]
\caption{Quantitative results of (\textit{top half}) re-exposed LDR-NVS on real scenes of HDR-NeRF and HDR-Plenoxels datasets and (\textit{bottom half}) HDR-NVS and re-exposed LDR-NVS on synthetic scenes of HDR-NeRF dataset. {\setlength{\fboxsep}{2pt}\colorbox{red!30}{red}} and {\setlength{\fboxsep}{2pt}\colorbox{yellow!30}{yellow}} denote the best and second-best results, respectively.} 
\label{tab:performance}
\begin{center}
\begin{tabular}{l|ccc|ccc}
\toprule
\multirow{2}{*}{Method}
& \multicolumn{3}{c|}{HDR-NeRF-Real -- LDR}
& \multicolumn{3}{c}{HDR-Plenoxels-Real -- LDR} \\
\cmidrule(lr){2-4} \cmidrule(lr){5-7}
& PSNR $\uparrow$ & SSIM $\uparrow$ & LPIPS $\downarrow$
& PSNR $\uparrow$ & SSIM $\uparrow$ & LPIPS $\downarrow$ \\
\midrule
HDR-GS
& 11.10 & 0.5212 & 0.370 & 4.70 & 0.2799 & 0.456 \\
GaussHDR
& \cellcolor{yellow!30}16.28 & \cellcolor{yellow!30}0.7849 & \cellcolor{yellow!30}0.176 & \cellcolor{yellow!30}17.73 & \cellcolor{yellow!30}0.7141 & \cellcolor{yellow!30}0.197 \\
Mono-HDR-GS
& 11.43 & 0.6627 & 0.301 & 9.79 & 0.5570 & 0.389 \\
SeHDR
& 13.75 & 0.6991 & 0.182 & 14.21 & 0.6644 & 0.232 \\
\textbf{PriorGS (Ours)}
& \cellcolor{red!30}24.31 & \cellcolor{red!30}0.9248 & \cellcolor{red!30}0.064 & \cellcolor{red!30}23.00 & \cellcolor{red!30}0.9021 & \cellcolor{red!30}0.092 \\

\midrule
\multirow{2}{*}{Method}
& \multicolumn{3}{c|}{HDR-NeRF-Syn -- HDR} & \multicolumn{3}{c}{HDR-NeRF-Syn -- LDR} \\
\cmidrule(lr){2-4} \cmidrule(lr){5-7}
& PSNR $\uparrow$ & SSIM $\uparrow$ & LPIPS $\downarrow$
& PSNR $\uparrow$ & SSIM $\uparrow$ & LPIPS $\downarrow$ \\
\midrule
HDR-GS
& 26.11 & \cellcolor{yellow!30}0.8939 & \cellcolor{yellow!30}0.113 & 10.86 & 0.3429 & 0.397 \\
GaussHDR
& \cellcolor{yellow!30}26.96 & 0.8828 & 0.114 & 11.99 & 0.5848 & 0.302  \\
Mono-HDR-GS
& 16.71 & 0.5642 & 0.359 & 12.44 & 0.5640 & 0.422 \\
SeHDR
& 15.14 & 0.6405 & 0.262 & \cellcolor{yellow!30}15.76 & \cellcolor{yellow!30}0.7450 & \cellcolor{yellow!30}0.169 \\
\textbf{PriorGS (Ours)}
& \cellcolor{red!30}31.22 & \cellcolor{red!30}0.9569 & \cellcolor{red!30}0.038 & \cellcolor{red!30}22.96 & \cellcolor{red!30}0.8937 & \cellcolor{red!30}0.081 \\
\bottomrule
\end{tabular}
\end{center}
\end{table}
\textbf{Datasets \& Implementation Details.} We evaluate PriorGS on two public datasets: HDR-NeRF \citep{HDR_Nerf} and HDR-Plenoxels \citep{hdr-plenoxels}. Each dataset contains 4 real scenes without HDR GT and captured with digital cameras. HDR-NeRF additionally provides 8 synthetic scenes created in Blender \citep{blender} with HDR GT. Note that HDR GTs are not involved in training of any method and are exclusively used for evaluation. In each scene, each view is available at 5 exposures $\{\Delta t_i\}^5_{i=1}$. We use a single-exposure training strategy, where we train all methods on one exposure and evalute re-exposed LDR-NVS at the remaining 4 exposures along with HDR-NVS. We repeat this protocol for all 5 training exposures and report the average results. We use 18 training and 17 test views per HDR-NeRF scene and 27 training and 13 test views per HDR-Plenoxels scene. We initialize the scenes using COLMAP \citep{colmap} and train them for 30K iterations on an RTX 3090 GPU. We select IntrinsicHDR \citep{intrinsic_hdr} as our 2D HDR reconstruction method for its robustness. Appendix \ref{subsec:apx_exp_details} provides further details.

\textbf{Metrics \& Baselines}. We evaluate HDR-NVS and re-exposed LDR-NVS using PSNR, SSIM \citep{ssim}, and LPIPS \citep{lpips}. In HDR quantitative comparisons, we follow \citet{sihdr-caveat} and apply luminance correction once per scene to adjust the scale ambiguity inherent in single-exposure HDR reconstruction \citep{Cao_2024_CVPR}. Next, following \citet{HDRGS}, we apply $\mu$-law range compression and then compute HDR metrics. We omit both steps for HDR qualitative comparisons and visualize HDR images after only a fixed tone-mapping. We compare PriorGS with all existing single-exposure methods Mono-HDR-GS \citep{MonoHDR} and SeHDR \citep{SeHDR}, and multi-exposure methods GaussHDR \citep{GaussHDR} and HDR-GS \citep{HDRGS} adapted to single-exposure setting. For fairness to other methods, we use GaussHDR’s 3DGS variant \citep{GaussianSplats} and assess its Scaffold-GS \citep{scaffold-GS} variant in Appendix \ref{subsec:apx_quant_results}. As SeHDR is closed-source, we reproduce it to our best ability. Appendix \ref{subsec:apx_exp_details} offers further details.

\subsection{Performance Comparison \& Application} \label{subsec:perf_application}
\textbf{Quantitative.} Table \ref{tab:performance} summarizes the quantitative results. PriorGS consistently outperforms all baselines by large margins across both datasets, real and synthetic scenes, and both HDR-NVS and re-exposed LDR-NVS. This shows the effectiveness of our LDR-to-HDR prior integration with closed-loop reconstruction, base–gain representation, and regularized CRF modeling. Appendix \ref{subsec:apx_quant_results} offers more results, including PriorGS’s robustness to the choice of 2D HDR reconstruction method.

\textbf{Qualitative.} Figure \ref{fig:hdr_qualitative} presents qualitative single-exposure HDR-NVS results from low-, medium-, and high-exposure inputs across real and synthetic scenes. Without learned LDR-to-HDR reconstruction priors, all baselines fail to recover scene radiance faithfully. However, PriorGS leverages this prior in a closed-loop reconstruction framework to produce substantially more accurate HDR reconstructions across all scenes and exposure settings. Moreover, in Figure ~\ref{fig:ldr-qualitative}, we present qualitative re-exposed LDR-NVS results across real and synthetic scenes under diverse input-output exposure setups. As observed, the flexible MLP-based CRFs of GaussHDR and Mono-HDR-GS generalize poorly to unseen exposures, producing overexposed or underexposed renderings and color artifacts. SeHDR mitigates this by heuristically fixing a CRF per scene but still suffers from inaccurate CRF estimation and limited scene adaptability. PriorGS instead regularizes CRF modeling with a learnable prior that guarantees monotonicity and promotes plausible responses, thereby achieving much more accurate re-exposure at unseen exposures and consistently outperforming all baselines. More qualitative results are offered in Appendix \ref{subsec:apx_qual_results}.
\begin{figure}[t]
    \begin{center}
    \includegraphics[width=0.95\linewidth]{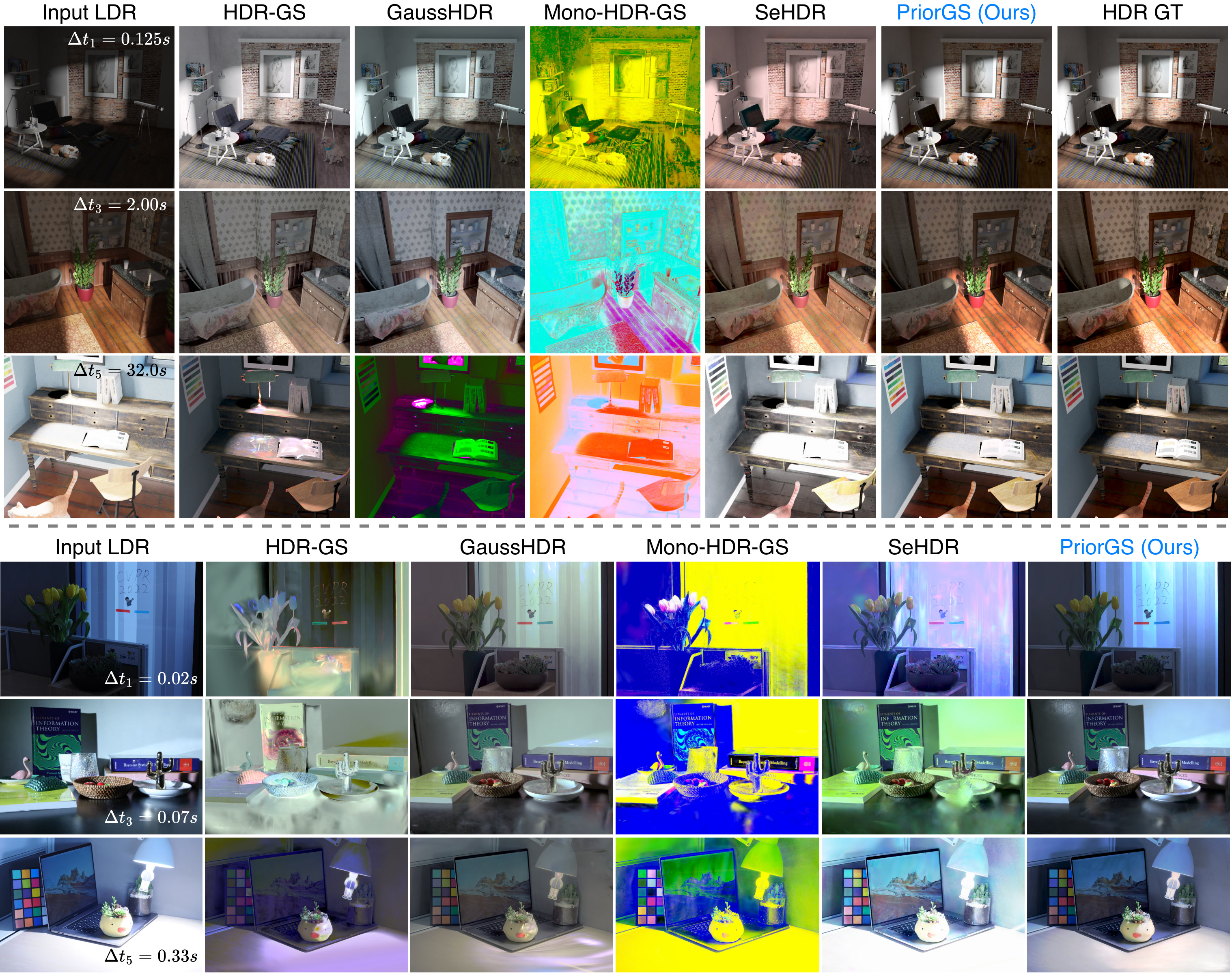}
    \end{center}
    \caption{Qualitative HDR-NVS results given low- ($\Delta t_1$), medium- ($\Delta t_3$), and high-exposure ($\Delta t_5$) inputs across (\textit{top half}) synthetic and (\textit{bottom half}) real scenes. PriorGS substantially outperforms all baselines in high-fidelity HDR reconstruction across all scenes and exposure settings.}
    \label{fig:hdr_qualitative}
\end{figure}
\begin{figure}[h]
    \begin{center}
    \includegraphics[width=0.95\linewidth]{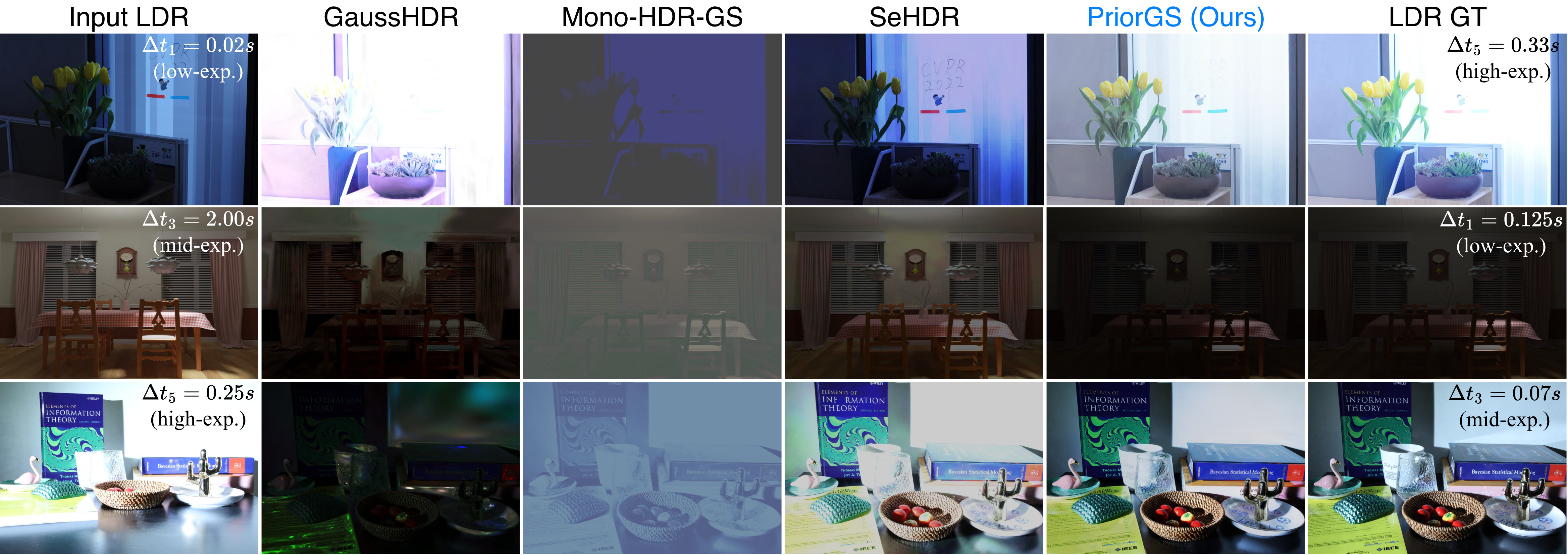}
    \end{center}
    \caption{Qualitative re-exposed LDR-NVS results across scenes for three input $\rightarrow$ output exposure setups: (\textit{top row}) $\Delta t_1 \rightarrow \Delta t_5$, (\textit{middle row}) $\Delta t_3 \rightarrow \Delta t_1$, (\textit{bottom row}) $\Delta t_5 \rightarrow \Delta t_3$. PriorGS consistently outperforms all baselines in novel-view LDR re-exposure at unseen exposures.} 
    \label{fig:ldr-qualitative}
\end{figure}

\textbf{3D HDR Appearance Editing.} Beyond performance improvement, our base-gain HDR representation enables continuous post-training control over HDR appearance in 3D without altering the LDR base scene. We introduce a coefficient $\rho \in \R_{\geq0}$ that scales the 3D gain field: each Gaussian’s log-gain SH coefficients $\bm{\beta}_i^g$ are scaled to $\rho\bm{\beta}_i^g$. This is equivalent to scaling evaluated log-gain $\vg_i$ to $\rho\vg_i$ in Equation~\ref{eq:gain_eq} as SH evaluation is linear in SH coefficients. Varying $\rho$ then controls HDR radiance enhancement relative to the LDR base: $\rho=0$ removes the gain and recovers the LDR base, $0<\rho<1$ yields intermediate enhancement, $\rho=1$ recovers the full HDR radiance, and $\rho>1$ extrapolates the enhancement for creative HDR appearance editing. In Figure \ref{fig:hdr-edit}, we vary $\rho$ from $0.0$ to $1.5$, report the corresponding dynamic range (DR), and show separate tone-mapped and low-exposure views of a novel HDR view. We observe that increasing $\rho$ expands the DR and enhances the HDR appearance and chromatic contrast. Notably, $\rho=0.5$ recovers highlight texture missing at $\rho=0$. Furthermore, as shown in Figure \ref{fig:hdr-edit}, the edits are inherently view-consistent as $\rho$ operates directly on 3D Gaussians, offering a key advantage over independent per-view 2D editing. Note that, without a 3D gain field, the same editing capability cannot be achieved by scaling Gaussians' HDR radiance, as this merely amounts to an exposure shift and does not expand the DR (see Figure \ref{fig:hdr-hist}). Our base–gain representation thus offers a unique mechanism for 3D HDR appearance editing while preserving the LDR-base scene. Appendix \ref{subsec:apx_application} offers extension to local edits and more scenes.
\begin{figure}[t]
    \begin{center}
    \includegraphics[width=0.95\linewidth]{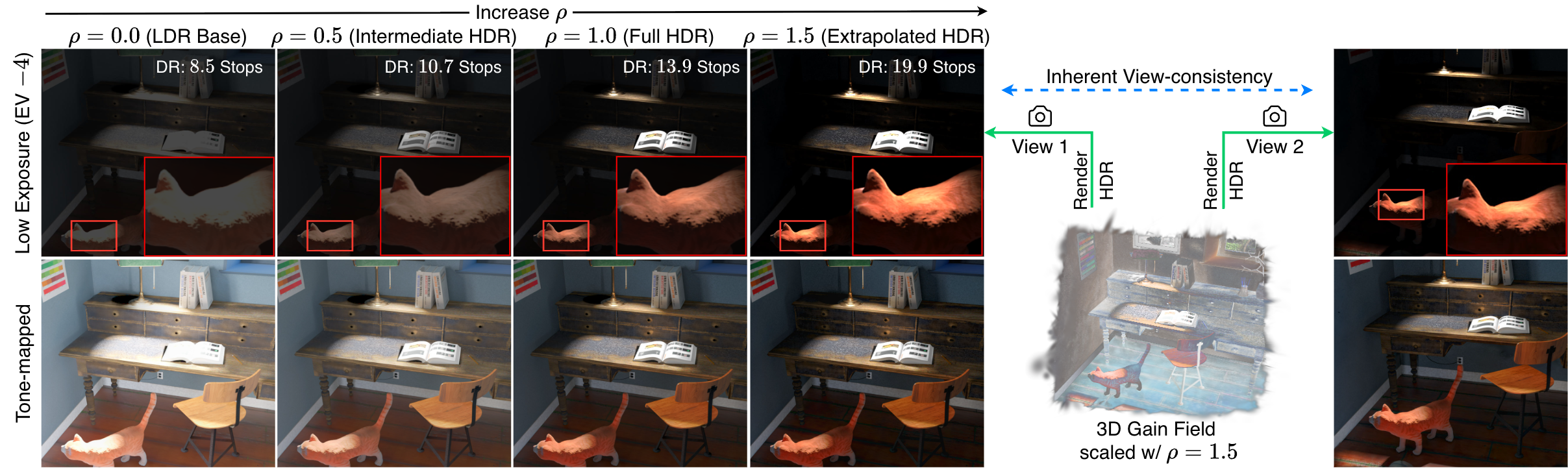}
    \end{center}
    \caption{View-consistent 3D HDR appearance editing via gain field scaling with $\rho$. We sweep $\rho$ and show an HDR view at a low-exposure (\textit{top row}) and after a fixed tone-mapping (\textit{bottom row}). Increasing $\rho$ expands dynamic range (DR) and enhances HDR appearance and chromatic contrast.}
    \label{fig:hdr-edit}
\end{figure}

\subsection{Ablation Studies}

We ablate PriorGS and report qualitative results using all datasets and quantitative results on real HDR-NeRF and HDR-Plenoxels scenes, whose challenging scenes and diverse capture settings better reveal the contribution of components. Appendix \ref{subsec:apx_abl_study} offers further ablation results and details.

\textbf{LDR-to-HDR Prior \& Closed-Loop Reconstruction.} In Table \ref{tab:ablation}, we ablate (a) LDR-to-HDR reconstruction prior by removing pseudo-HDR supervision; and (b) closed-loop reconstruction by blocking gradients from the re-exposure loss to HDR representation, such that this loss optimizes only the CRF and does not impact HDR recovery. Both ablations degrade the performance, with removal of LDR-to-HDR prior causing a substantial drop. Together with qualitative results in Figures \ref{fig:main-abl} and \ref{fig:apx_abl_hdr_rec}, these results prove the necessity of learned LDR-to-HDR priors for faithful HDR-NVS from single-exposure views and the contribution of closed-loop reconstruction to HDR-NVS quality.

\textbf{Base-Gain HDR Representation.} As shown in Table \ref{tab:ablation}, we ablate our base-gain HDR representation using (c) direct HDR encoding with equal capacity (bigger SH degree); and MLP replacements for the gain field that predict Gaussians' HDR color from either (d) base color $\vc$, position $\vmu$, and view direction $\vd$ or (e) base color $\vc$ alone. All ablations degrade the overall performance, with direct encoding causing the largest drop. Qualitatively, Figures \ref{fig:main-abl} and \ref{fig:apx_abl_direct_hdr} show that direct HDR encoding exhibits severe color and geometry artifacts, while Figures \ref{fig:main-abl} and \ref{fig:apx_abl_gain_to_mlp} reveal failure of MLPs in reconstructing high-frequency textures in challenging regions. Also, all configurations preclude 3D HDR editing enabled by the gain field. These results show the benefits of our base-gain representation.

\textbf{Prior-Constrained CRF Modeling.} In Table \ref{tab:ablation}, we replace our CRF $f_\vtheta$ with (f) an MLP applied in 3D on log-exposed per-Gaussian HDR colors; (g) an MLP applied in 2D on log-exposed HDR renderings; and (h) fixed gamma encoding $\gamma=2.2$. Our prior-constrained CRF consistently outperforms flexible MLPs and the fixed CRF, validating the benefit of regularizing CRF with a learnable prior in single-exposure HDR-NVS. Appendix \ref{subsec:apx_abl_study} provides qualitative comparisons.

\begin{table}[t]
\caption{Ablations on re-exposed LDR-NVS on real scenes of HDR-NeRF and HDR-Plenoxels, with best and second-best results denoted by {\setlength{\fboxsep}{2pt}\colorbox{red!30}{red}} and  {\setlength{\fboxsep}{2pt}\colorbox{yellow!30}{yellow}}. $x\rightarrow y$ denotes replacing $x$ with $y$.}
\label{tab:ablation}
\begin{center}
\setlength{\tabcolsep}{5pt}
\begin{tabular}{@{}l|ccc|ccc}
\toprule
\multirow{2}{*}{(ID) Description}
& \multicolumn{3}{c|}{HDR-NeRF-Real -- LDR}
& \multicolumn{3}{c}{HDR-Plenoxels-Real -- LDR} \\
\cmidrule(lr){2-4} \cmidrule(lr){5-7}
& PSNR$\uparrow$ & SSIM$\uparrow$ & LPIPS$\downarrow$
& PSNR$\uparrow$ & SSIM$\uparrow$ & LPIPS$\downarrow$ \\
\midrule
(a) Remove LDR-to-HDR prior
& 18.39 & 0.8503 & 0.099 & 15.00 & 0.7881 & 0.141 \\
(b) Remove reconstruction loop
& 23.52 & 0.9180 & 0.073 & \cellcolor{yellow!30}22.97 & \cellcolor{yellow!30}0.9010 & 0.099 \\
\midrule
(c) Base-Gain $\rightarrow$ direct encoding
& 22.71 & 0.8981 & 0.105 & 15.05 & 0.6860 & 0.251 \\
(d) Gain Field $\rightarrow$ $\operatorname{MLP}(\vc, \vmu, \vd)$
& \cellcolor{red!30}24.44 & 0.9238 & \cellcolor{yellow!30}0.066 & 22.64 & 0.8979 & \cellcolor{yellow!30}0.093 \\
(e) Gain Field $\rightarrow$ $\operatorname{MLP}(\vc)$
& 24.28 & \cellcolor{red!30}0.9268 & \cellcolor{yellow!30}0.066 & 22.07 & 0.8732 & 0.111 \\
\midrule
(f) CRF $f_\vtheta$ $\rightarrow$ $\operatorname{MLP}$ in 3D
& 23.26 & 0.8874 & 0.084 & 20.35 & 0.8296 & 0.147 \\
(g) CRF $f_\vtheta$ $\rightarrow$ $\operatorname{MLP}$ in 2D
& 23.53 & 0.9071 & 0.075 & 22.34 & 0.8802 & 0.126 \\
(h) CRF $f_\vtheta$ $\rightarrow$ fixed gamma
& 14.87 & 0.6784 & 0.222 & 12.57 & 0.5031 & 0.320 \\
\midrule
\textbf{PriorGS (Ours)}
& \cellcolor{yellow!30}24.31 & \cellcolor{yellow!30}0.9248 & \cellcolor{red!30}0.064 & \cellcolor{red!30}23.00 & \cellcolor{red!30}0.9021 & \cellcolor{red!30}0.092 \\
\bottomrule
\end{tabular}
\end{center}
\end{table}
\begin{figure}[t]
    \begin{center}
    \includegraphics[width=0.95\linewidth]{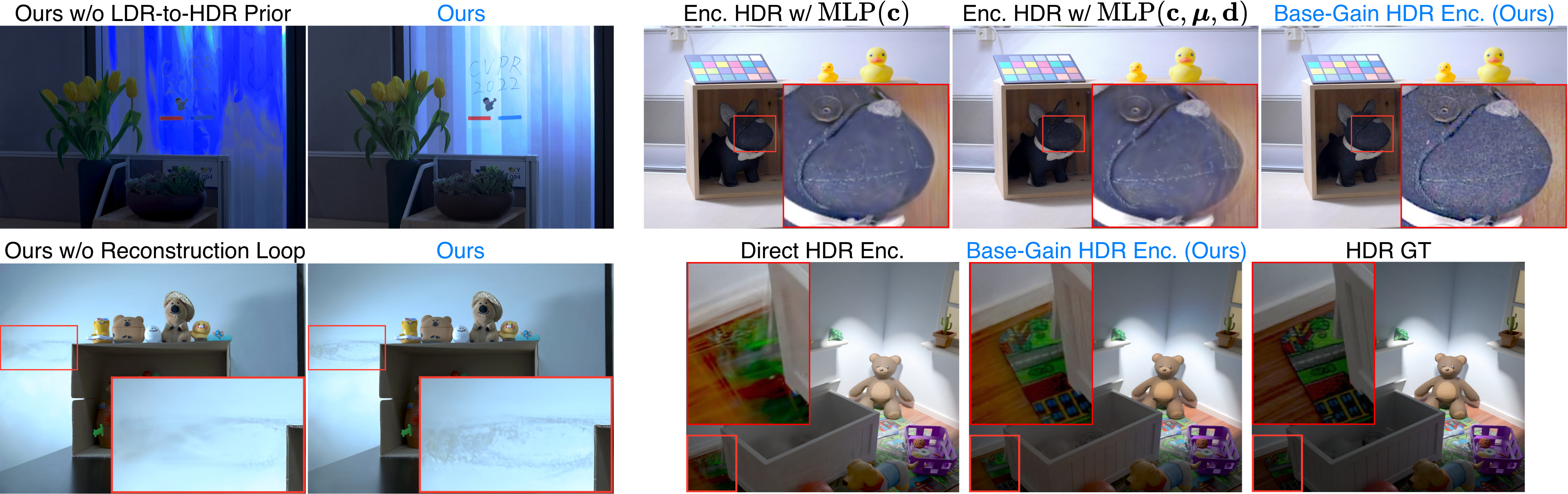}
    \end{center}
    \caption{Ablations on HDR views for (\textit{top left}) removing LDR-to-HDR prior; (\textit{bottom left}) removing reconstruction loop; and replacing our base-gain HDR encoding by (\textit{top right}) two MLP variants or (\textit{bottom right}) direct HDR encoding. All ablations degrade HDR reconstruction quality.}
    \label{fig:main-abl}
\end{figure}
\section{Discussion \& Conclusion}
We present a novel single-exposure HDR-NVS method based on 3DGS that leverages LDR-to-HDR reconstruction priors within a closed reconstruction loop, encodes HDR by a novel 3D base-gain representation, and regularizes CRF modeling with a learnable prior to achieve high-fidelity HDR-NVS and LDR re-exposure. We prove superiority of our method with extensive experiments and present its application in 3D HDR appearance editing. A few limitations present room for improvement. Our method is limited by the quality of pseudo-HDR targets generated by 2D HDR reconstruction methods: poor recovery of lost textures in the targets can degrade our HDR reconstruction. Also, while our base-gain representation improves robustness to cross-view inconsistencies in the targets, major discrepancies may still be encoded as fake view-dependent appearance, reducing scene realism. Future directions thus include (1) generative 3D HDR reconstruction that infers radiance directly in 3D using neighboring-view cues to improve reconstruction and cross-view consistency; and (2) mitigating inconsistency issue by uncertainty learning \citep{nerf-in-the-wild}.

\subsection*{AI use statement}
In this work, we used generative AI tools to implement methods, and assist in the writing of proofs. We have not used generative AI tools to generate synthetic data sets, help develop theoretical models or conceptual frameworks, formulate mathematical claims, provide critical ingredients for proving mathematical claims, propose or refine hypotheses, design or provide feedback on research methodology or experiments, assist with translation, clean and reformat dataset, support qualitative and thematic data analysis, interpret results, draft parts of the paper, and create or modify scientific figures or images. 
Additionally, we used generative AI tools for polishing the writing and enhancing readability, sourcing/searching for information, and brainstorming. We have reviewed all AI-assisted work in the following forms: (1) LLM-generated code was verified and tested for correctness by at least one author; (2) All LLM-generated text used for polishing the writing of the paper and the monotonicity proof was reviewed by at least two authors and manually refined repeatedly to ensure the intended message is delivered; and (3) All LLM-driven search results were manually reviewed to ensure factual accuracy. We take responsibility for the final content of this work, including text, claims or artifacts produced with the aid of generative AI.

\subsection*{Reproducibility statement}
Sections \ref{sec:prop_method} and \ref{sec:exp_settings}, together with Appendix \ref{subsec:apx_exp_details}, provide the methodological and implementation details needed to reproduce our results. Our implementation builds on the publicly available HDR-GS implementation \citep{HDRGS}. We use the publicly available HDR-NeRF \citep{HDR_Nerf} and HDR-Plenoxels \citep{hdr-plenoxels} datasets for training and evaluation. Pseudo-HDR targets are generated from their input LDR views using the publicly available 2D HDR reconstruction method IntrinsicHDR \citep{intrinsic_hdr}. Our implementation will be made publicly available on GitHub upon publication. All experiments can be reproduced using a single NVIDIA RTX 3090 GPU with 24GB of memory.


\subsubsection*{Acknowledgments}
This work was supported in part by Natural Sciences and Engineering Research Council of Canda under Discovery Grant (RGPIN-2024-05154) and Canada Research Chair Program (CRC-2025-00266).

\bibliography{iclr2027_conference}
\bibliographystyle{iclr2027_conference}

\clearpage

\appendix

\section{Appendix}
\subsection{Discussion on Base-Gain HDR Representation} \label{subsec:apx_base_gain}
Our base-gain HDR representation anchors the base appearance and shared geometry to reliable LDR supervision and directs cross-view inconsistencies in pseudo-HDR targets toward the gain field. This provides a soft inductive bias that discourages, but does not strictly prevent, HDR supervision from updating the base appearance and geometry. We made this design choice after evaluating the following hard constraints:
\begin{itemize}
    \item \textbf{Preventing HDR supervision from updating geometry.} In regions where LDR views are degraded by shadow noise or highlight clipping, pseudo-HDR targets may contain recovered textures. 3D reconstruction of these textures may require geometric refinement beyond what LDR supervision alone can provide. We therefore allow HDR supervision to update the shared geometry, a choice motivated by the above rationale and our empirical observations.
    \item \textbf{Preventing HDR supervision from updating base appearance.} We evaluated this constraint by implementing stop gradient on linearized LDR-base color $\vl_i^v$ in Equation \ref{eq:gain_eq} and achieved no improvement in HDR reconstruction quality. We therefore favor the simpler design, avoiding an additional constraint that provides no observed benefit.
\end{itemize}

\subsection{Representational Capacity Analysis \& Monotonicity Proof of Our CRF} 
\label{subsec:apx_crf_analysis}
\subsubsection{Representational Capacity Analysis}
We evaluate the representational capacity of our CRF model (see Section \ref{subsec:crf_model}) on all 201 camera response curves in DoRF \citep{EMoR}. Specifically, we follow DoRF's range-normalization convention to fit our CRF's normalized response to each DoRF curve by minimizing the mean squared error (MSE).
Figure \ref{fig:crf_cover} compares the DoRF response curves with the resulting fitted responses. Qualitatively, the fitted responses closely reproduce the broad diversity of DoRF curves, yielding an overall response family whose distribution and envelope are nearly indistinguishable from those of the DoRF curves. Quantitatively, across all curves, our CRF model achieves a mean MSE of $2.081\times10^{-5}$, a median MSE of $4.328\times10^{-6}$, and a maximum MSE of $8.225\times10^{-4}$. For comparison, fitting the direct two-parameter response $t_{\bm\phi}(x)$ (see Equation~\ref{eq:crf_model}) and a gamma response $x^\gamma$ yields mean MSEs of $5.717\times10^{-5}$ and $3.801\times10^{-3}$, respectively. Thus, our model reduces the mean MSE by $63.6\%$ relative to the two-parameter response and by $99.45\%$ relative to the gamma model. This demonstrates that our model represents empirical camera responses substantially more accurately than the baselines.
\begin{figure}[h]
    \centering

    \begin{subfigure}[b]{0.48\linewidth}
        \centering
        \includegraphics[width=\linewidth]{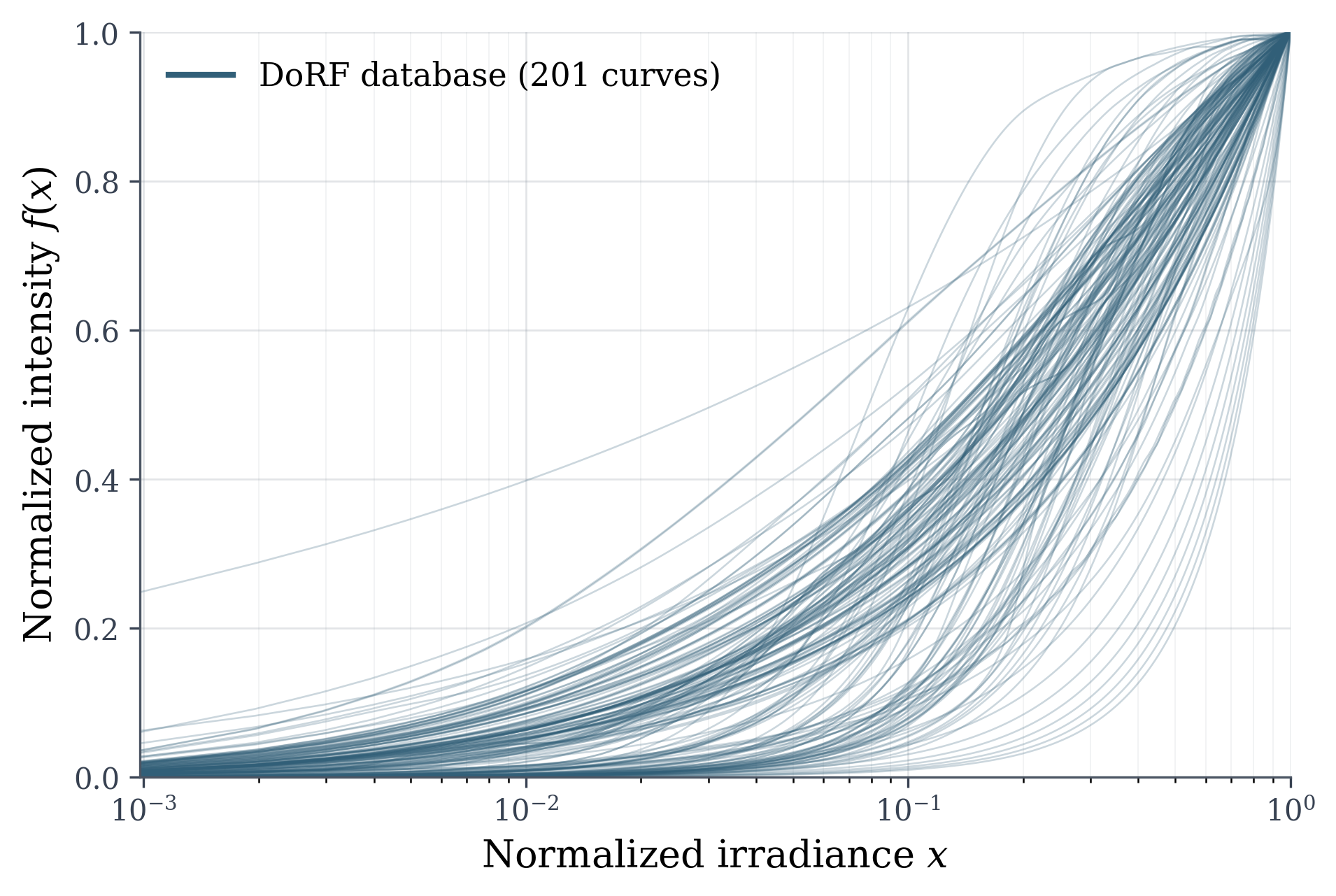}
        \caption{DoRF responses}
        \label{fig:dorf_cover}
    \end{subfigure}
    \hfill
    \begin{subfigure}[b]{0.48\linewidth}
        \centering
        \includegraphics[width=\linewidth]{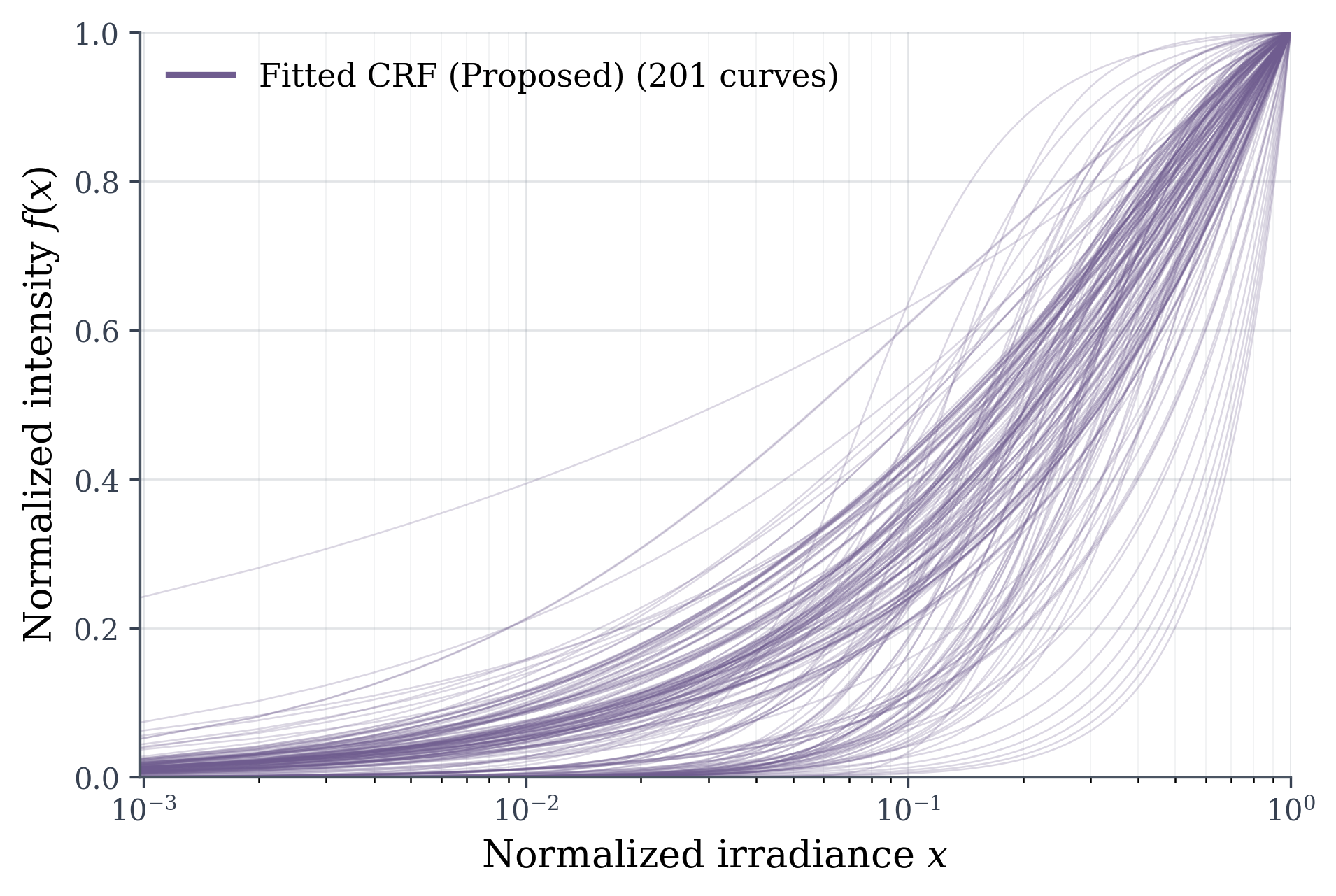}
        \caption{Best-fitting responses of our CRF}
        \label{fig:our_cover}
    \end{subfigure}

    \caption{Representational capacity of our proposed CRF model on the DoRF database. (a) the 201 DoRF response curves. (b) the independently fitted responses of our CRF. All curves are displayed using a logarithmic irradiance axis; the zero-irradiance samples are included during fitting but omitted from the visualization.}
    \label{fig:crf_cover}
\end{figure}

\subsubsection{Monotonicity Proof}
As characterized by \citet{EMoR}, a physically valid CRF should preserve the ordering of input irradiance: a greater irradiance must produce a bigger recorded pixel intensity. We show below that our proposed CRF satisfies this: It is strictly increasing over its domain by construction. Our proposed CRF is defined as
\begin{equation}
\begin{gathered}
    f(E)
    = t\left(\eta(E)\right), \\
    t(x)
    = (1+\sigma)\frac{x^n}{x^n+\sigma},
    \qquad
    \eta(E)
    = \operatorname{sigmoid}\!\left(a\ln(E+\epsilon)+b\right).
\end{gathered}
\end{equation}
where $E\in\mathbb{R}_{\geq 0}$,
$\sigma,n,a\in\mathbb{R}_{>0}$, $b\in\mathbb{R}$, and $\epsilon = 1\times10^{-8}$.

\begin{lemma}[Composition of strictly increasing functions]
\label{lem:monotone_composition}
Let $u\colon I\to J$ and $v\colon J\to\mathbb{R}$ be strictly
increasing. Then $v\circ u$ is strictly increasing.
\end{lemma}

\begin{proof}
Let $x_1,x_2\in I$ with $x_1<x_2$. Since $u$ is strictly increasing,
\[
u(x_1)<u(x_2).
\]
Because $u(x_1),u(x_2)\in J$ and $v$ is strictly increasing, it follows that
\[
v(u(x_1))<v(u(x_2)).
\]
Thus,
\[
(v\circ u)(x_1)<(v\circ u)(x_2),
\]
so $v\circ u$ is strictly increasing on $I$.
\end{proof}

\begin{proposition}
\label{prop:crf_strict_monotonicity}
Under the parameter constraints above, $f$ is strictly
increasing on $\mathbb{R}_{\geq 0}$.
\end{proposition}

\begin{proof}
Define $h(E)=a\ln(E+\epsilon)+b.$
Since $E+\epsilon>0$ and $a>0$,
\begin{equation}
h'(E)=\frac{a}{E+\epsilon}>0,
\end{equation}
so $h$ is strictly increasing on $\mathbb{R}_{\geq 0}$. Now, since sigmoid has range $(0, 1)$,
\begin{equation}
\operatorname{sigmoid}'(t)
=\operatorname{sigmoid}(t)\bigl(1-\operatorname{sigmoid}(t)\bigr)>0.
\end{equation}
Therefore, sigmoid is also strictly increasing, and thus 
$\eta=\operatorname{sigmoid}\circ h$ is strictly increasing due to Lemma \ref{lem:monotone_composition} and has range $(0,1)$. Finally, for every $x>0$,
\begin{equation}
t'(x)
=\frac{(1+\sigma)n\sigma x^{n-1}}
       {(x^n+\sigma)^2}>0,
\end{equation}
since $\sigma,n>0$. Hence $t$ is strictly increasing on
$(0,\infty)$. Since $\eta(E)\in(0,1)$,
$f=t\circ\eta$ is a composition of
strictly increasing functions. Lemma~\ref{lem:monotone_composition}
thus shows that $f$ is strictly increasing on
$\mathbb{R}_{\geq 0}$.
\end{proof}

\subsection{Experimental Setup Details} \label{subsec:apx_exp_details}
\textbf{Details of Datasets.} The real and synthetic HDR-NeRF \citep{HDR_Nerf} views have native resolutions of $3216\times2136$ and $800\times800$, respectively. We train and evaluate in the native resolution for synthetic scenes and downsample real scenes to $1608\times1068$ ($\frac{1}{2}$ scale) to reduce computational cost. Similarly, we downsample the HDR-Plenoxels \citep{hdr-plenoxels} views from their native resolution of $5952\times4480$ to $\frac{1}{4}$ scale, resulting in a resolution of $1488\times1120$ for training and evaluation. In the HDR-Plenoxels dataset, the \textit{plant} and \textit{desk} scenes are available at all 5 exposures, $\{\Delta t_1,\Delta t_2,\Delta t_3,\Delta t_4,\Delta t_5\}$, whereas \textit{coffee} and \textit{character} scenes provide only 3 exposures $\{\Delta t_1,\Delta t_3,\Delta t_5\}$. Accordingly, we conduct five single-exposure training runs for \textit{plant} and \textit{desk} but only three for \textit{coffee} and \textit{character}. Reported results are averaged across runs for all training exposures and all scenes of a dataset.

\textbf{Further Implementation Details.} The Adam optimizer \citep{adam} is used for training. We adopt the learning-rate schedules and densification/pruning hyperparameters of \citet{GaussianSplats} for training the Gaussian scene representations and use the same configuration across all baselines. We initialize our CRF model $f_\vtheta$ to the mean DoRF response \citep{EMoR, HDRCNN} and optimize its parameters with an initial learning rate of $5\times10^{-4}$ that exponentially decays to $5\times10^{-5}$. To ensure that CRF parameters $\sigma$, $n$, and $a$ remain strictly positive throughout training, we use exponential reparameterizations: $\sigma = \exp(\tilde \sigma)$, $n = \exp(\tilde n)$, $a = \exp(\tilde a)$ and optimize the unconstrained log-parameters $\tilde\sigma,\tilde n,\tilde a \in \R$ jointly with $b$. We note that our CRF model $f_\vtheta$ learns a single response curve shared across all three RGB channels. To improve optimization stability, we normalize pseudo-HDR targets to $[0, 1]$ by dividing them by the maximum radiance across all pseudo-HDR targets in the scene.

\textbf{Further Details of Metrics \& Baselines.} As discussed in Section \ref{sec:exp_settings}, in HDR quantitative comparisons, we follow the recommendation of \citet{sihdr-caveat} and apply luminance correction to adjust the scale ambiguity inherent in single-exposure HDR reconstruction \citep{Cao_2024_CVPR} and enhance judgment of HDR metrics. To adapt their proposed correction to 3D settings, we compute luminance correction weights once per scene using its first indexed viewpoint and then reuse the weights across all evaluation viewpoints. We note that unlike per-view correction, this per-scene correction avoids masking of intensity flickering among views during evaluation. After luminance correction, we apply $\mu$-law range compression ($\mu=5000$) and compute HDR metrics in this tone-mapped domain, following prior works \citep{HDRGS, GaussHDR}. For comparisons with multi-exposure baselines, we adapt the publicly available implementations of GaussHDR \citep{GaussHDR} and HDR-GS \citep{HDRGS} to run on single-exposure inputs. The recent 3DGS-based multi-exposure methods PhysHDR-GS \citep{PhysHDR} and Expo-GS \citep{expo-gs} are not included only because their implementations are not publicly available. All baselines, whether single-exposure or multi-exposure, were retrained given our HDR GT-free single-exposure training protocol (as explained in Section \ref{sec:exp_settings}) to ensure fairness in comparisons.

\begin{table}[t]
\caption{PriorGS's performance after integration with different 2D HDR reconstruction methods on (\textit{top half}) re-exposed LDR-NVS on real scenes of HDR-NeRF and HDR-Plenoxels datasets and (\textit{bottom half}) HDR-NVS and re-exposed LDR-NVS on synthetic scenes of HDR-NeRF dataset. {\setlength{\fboxsep}{2pt}\colorbox{red!30}{red}} and {\setlength{\fboxsep}{2pt}\colorbox{yellow!30}{yellow}} denote the best and second-best results, respectively.} 
\label{tab:robust_to_ldr2hdr}
\begin{center}
\setlength{\tabcolsep}{3.5pt}
\begin{tabular}{@{}l|ccc|ccc}
\toprule
\multirow{2}{*}{Method}
& \multicolumn{3}{c|}{HDR-NeRF-Real -- LDR}
& \multicolumn{3}{c}{HDR-Plenoxels-Real -- LDR} \\
\cmidrule(lr){2-4} \cmidrule(lr){5-7}
& PSNR$\uparrow$ & SSIM$\uparrow$ & LPIPS$\downarrow$
& PSNR$\uparrow$ & SSIM$\uparrow$ & LPIPS$\downarrow$ \\
\midrule
GaussHDR
& 16.28 & 0.7849 & 0.176 & 17.73 & 0.7141 & 0.197 \\
SeHDR
& 13.75 & 0.6991 & 0.182 & 14.21 & 0.6644 & 0.232 \\
\textbf{PriorGS + DeepHDR}
& 20.38 & 0.8604 & \cellcolor{yellow!30}0.085 & 18.43 & 0.8111 & 0.115 \\
\textbf{PriorGS + ExpandNet}
& \cellcolor{yellow!30}21.18 & \cellcolor{yellow!30}0.8771 & 0.094 & \cellcolor{yellow!30}21.23 & \cellcolor{yellow!30}0.8624 & \cellcolor{yellow!30}0.100 \\
\textbf{PriorGS + GMNet}
& 19.36 & 0.8482 & 0.095 & 16.94 & 0.7887 & 0.126 \\
\textbf{PriorGS + IntrinsicHDR (Primary)}
& \cellcolor{red!30}24.31 & \cellcolor{red!30}0.9248 & \cellcolor{red!30}0.064 & \cellcolor{red!30}23.00 & \cellcolor{red!30}0.9021 & \cellcolor{red!30}0.092 \\

\midrule
\multirow{2}{*}{Method}
& \multicolumn{3}{c|}{HDR-NeRF-Syn -- HDR} & \multicolumn{3}{c}{HDR-NeRF-Syn -- LDR} \\
\cmidrule(lr){2-4} \cmidrule(lr){5-7}
& PSNR$\uparrow$ & SSIM$\uparrow$ & LPIPS$\downarrow$
& PSNR$\uparrow$ & SSIM$\uparrow$ & LPIPS$\downarrow$ \\
\midrule
GaussHDR
& 26.96 & 0.8828 & 0.114 & 11.99 & 0.5848 & 0.302 \\
SeHDR
& 15.14 & 0.6405 & 0.262 & 15.76 & 0.7450 & 0.169 \\
\textbf{PriorGS + DeepHDR}
& 30.81 & 0.9512 & 0.052 & \cellcolor{yellow!30}23.76 & 0.8818 & 0.096 \\
\textbf{PriorGS + ExpandNet}
& 30.35 & 0.9437 & \cellcolor{yellow!30}0.039 & \cellcolor{red!30}24.06 & \cellcolor{red!30}0.9125 & \cellcolor{red!30}0.067 \\
\textbf{PriorGS + GMNet}
& \cellcolor{red!30}31.77 & \cellcolor{yellow!30}0.9540 & 0.043 & 23.64 & 0.8838 & 0.083 \\
\textbf{PriorGS + IntrinsicHDR (Primary)}
& \cellcolor{yellow!30}31.22 & \cellcolor{red!30}0.9569 & \cellcolor{red!30}0.038 & 22.96 & \cellcolor{yellow!30}0.8937 & \cellcolor{yellow!30}0.081 \\
\bottomrule
\end{tabular}
\end{center}
\end{table}

\subsection{Further Quantitative Results} \label{subsec:apx_quant_results}
\textbf{PriorGS's Robustness across 2D Single-exposure HDR Reconstruction Methods.} To evaluate PriorGS’s robustness to the choice of 2D single-exposure HDR reconstruction method, we replace our primary 2D HDR reconstruction method, IntrinsicHDR \citep{intrinsic_hdr}, with three alternatives: DeepHDR \citep{deephdr}, ExpandNet \citep{ExpandNet}, and GMNet \citep{gmnet}. Table \ref{tab:robust_to_ldr2hdr} summarizes the results for these alternative configurations and compares them with our closest baselines, GaussHDR \citep{GaussHDR} and SeHDR \citep{SeHDR}. We make the following observations: (1) Regardless of the 2D HDR reconstruction method integrated into PriorGS, every resulting configuration consistently outperforms both baselines across both datasets, including real and synthetic scenes, and on both HDR-NVS and re-exposed LDR-NVS, with a single exception: the GMNet configuration trails GaussHDR in PSNR on HDR-Plenoxels dataset while surpassing it by large margins in SSIM and LPIPS. These results further validate the effectiveness of PriorGS in single-exposure HDR-NVS and LDR re-exposure; (2) On synthetic scenes, the primary and all alternative 2D HDR reconstruction methods yield comparable PriorGS performance, demonstrating robustness to method choice. However, on the more challenging real scenes, which span diverse cameras and exposure settings, PriorGS exhibits more sensitivity to the selected reconstruction method; and (3) Our primary 2D HDR reconstruction method, IntrinsicHDR, yields the most robust performance among all evaluated methods: it outperforms all alternatives on the more challenging real scenes while achieving comparable results on synthetic scenes.

\textbf{LDR-NVS Performance at the Observed Exposure.} In addition to our primary evaluation of re-exposed LDR-NVS at unseen exposures in Table \ref{tab:performance}, we evaluate LDR-NVS quality at the observed (training) exposure. Our base-gain HDR representation enables joint reconstruction of an LDR-base at the input exposure and an HDR scene. This, therefore, allows PriorGS to render a novel LDR view at the training exposure through two pathways: directly from the LDR-base or by re-exposing the HDR rendering through the learned CRF. We opt to use the direct LDR-base rendering as our primary output when LDR view at training exposure is requested. The results averaged across all five single-exposure training runs for real and synthetic scenes are summarized in Tables \ref{tab:observed_ldr} and \ref{tab:observed_ldr2}, respectively. PriorGS outperforms all baselines on the real HDR-NeRF and HDR-Plenoxels scenes and performs comparably to Mono-HDR-GS \citep{MonoHDR} on the synthetic HDR-NeRF scenes. Together with Table \ref{tab:performance}, these results show that PriorGS not only generalizes substantially better to unseen exposures but also achieves excellent LDR-NVS performance at the observed exposure. This highlights the superiority of PriorGS and the benefit of our base-gain HDR representation in providing a clean LDR base scene while simultaneously maintaining a high-quality HDR scene representation.
\begin{table}[t]
\caption{Quantitative results on LDR-NVS at the training exposure on real scenes of HDR-NeRF and HDR-Plenoxels datasets. {\setlength{\fboxsep}{2pt}\colorbox{red!30}{red}} and {\setlength{\fboxsep}{2pt}\colorbox{yellow!30}{yellow}} denote the best and second-best results, respectively.}
\label{tab:observed_ldr}
\begin{center}
\begin{tabular}{l|ccc|ccc}
\toprule
\multirow{2}{*}{Method}
& \multicolumn{3}{c|}{HDR-NeRF-Real -- LDR}
& \multicolumn{3}{c}{HDR-Plenoxels-Real -- LDR} \\
\cmidrule(lr){2-4}
\cmidrule(lr){5-7}
& PSNR $\uparrow$ & SSIM $\uparrow$ & LPIPS $\downarrow$
& PSNR $\uparrow$ & SSIM $\uparrow$ & LPIPS $\downarrow$ \\
\midrule
HDR-GS
& 35.24 & \cellcolor{yellow!30}0.9667 & 0.0198 
& 30.43 & 0.9428 & 0.0492 \\
GaussHDR
& 34.42 & 0.9662 & 0.0366 
& 30.54 & 0.9447 & 0.0500 \\
Mono-HDR-GS
& \cellcolor{yellow!30}35.65 & 0.9656 & \cellcolor{yellow!30}0.0172 
& \cellcolor{yellow!30}31.02 & \cellcolor{yellow!30}0.9465 & \cellcolor{yellow!30}0.0464 \\
SeHDR
& 30.08 & 0.9632 & 0.0346 
& 21.77 & 0.9037 & 0.1157 \\
\textbf{PriorGS (ours)}
& \cellcolor{red!30}36.10 & \cellcolor{red!30}0.9700 & \cellcolor{red!30}0.0158 
& \cellcolor{red!30}31.71 & \cellcolor{red!30}0.9482 & \cellcolor{red!30}0.0439 \\
\bottomrule
\end{tabular}
\end{center}
\end{table}

\begin{table}[t]
\caption{Quantitative results on LDR-NVS at the training exposure on synthetic scenes of HDR-NeRF dataset. {\setlength{\fboxsep}{2pt}\colorbox{red!30}{red}} and {\setlength{\fboxsep}{2pt}\colorbox{yellow!30}{yellow}} denote the best and second-best results, respectively.}
\label{tab:observed_ldr2}
\begin{center}
\begin{tabular}{l|ccc}
\toprule
\multirow{2}{*}{Method}
& \multicolumn{3}{c}{HDR-NeRF-Syn -- LDR} \\
\cmidrule(lr){2-4}
& PSNR $\uparrow$ & SSIM $\uparrow$ & LPIPS $\downarrow$ \\
\midrule
HDR-GS
& 41.56 & 0.9825 & 0.0088 \\
GaussHDR
& 40.90 & 0.9795 & 0.0154 \\
Mono-HDR-GS
& \cellcolor{red!30}41.80 & \cellcolor{yellow!30}0.9828 & \cellcolor{red!30}0.0083 \\
SeHDR
& 32.53 & 0.9768 & 0.0199 \\
\textbf{PriorGS (ours)}
& \cellcolor{yellow!30}41.77 & \cellcolor{red!30}0.9836 & \cellcolor{yellow!30}0.0087 \\
\bottomrule
\end{tabular}
\end{center}
\end{table}

\textbf{Comparison Against Scaffold-GS variant of GaussHDR.} Beyond the baselines introduced in Section \ref{sec:exp_settings}, Table \ref{tab:scaffold_comp} compares PriorGS with Scaffold-GS variant of GaussHDR \citep{GaussHDR}. We exclude this variant from the main comparison as all other methods, including ours, use standard 3DGS representation \citep{GaussianSplats}. Despite GaussHDR’s superior Scaffold-GS representation, PriorGS consistently outperforms it on both datasets, across real and synthetic scenes, and in both HDR-NVS and re-exposed LDR-NVS, highlighting the effectiveness of our proposed method.

\begin{table}[t]
\caption{Quantitative comparison of PriorGS vs. the 3DGS \citep{GaussianSplats} and Scaffold-GS$\dagger$ \citep{scaffold-GS} variants of GaussHDR \citep{GaussHDR} for (\textit{top half}) re-exposed LDR-NVS on real scenes of HDR-NeRF and HDR-Plenoxels datasets and (\textit{bottom half}) HDR-NVS and re-exposed LDR-NVS on synthetic scenes of HDR-NeRF dataset. {\setlength{\fboxsep}{2pt}\colorbox{red!30}{red}} and {\setlength{\fboxsep}{2pt}\colorbox{yellow!30}{yellow}} denote the best and second-best results, respectively.} 
\label{tab:scaffold_comp}
\begin{center}
\begin{tabular}{l|ccc|ccc}
\toprule
\multirow{2}{*}{Method}
& \multicolumn{3}{c|}{HDR-NeRF-Real -- LDR}
& \multicolumn{3}{c}{HDR-Plenoxels-Real -- LDR} \\
\cmidrule(lr){2-4} \cmidrule(lr){5-7}
& PSNR $\uparrow$ & SSIM $\uparrow$ & LPIPS $\downarrow$
& PSNR $\uparrow$ & SSIM $\uparrow$ & LPIPS $\downarrow$ \\
\midrule
GaussHDR$\dagger$
& \cellcolor{yellow!30}22.76 & \cellcolor{yellow!30}0.8739 & \cellcolor{yellow!30}0.128 & \cellcolor{yellow!30}21.03 & \cellcolor{yellow!30}0.8635 & \cellcolor{yellow!30}0.115 \\
GaussHDR
& 16.28 & 0.7849 & 0.176 & 17.73 & 0.7141 & 0.197 \\
\textbf{PriorGS (Ours)}
& \cellcolor{red!30}24.31 & \cellcolor{red!30}0.9248 & \cellcolor{red!30}0.064 & \cellcolor{red!30}23.00 & \cellcolor{red!30}0.9021 & \cellcolor{red!30}0.092 \\

\midrule
\multirow{2}{*}{Method}
& \multicolumn{3}{c|}{HDR-NeRF-Syn -- HDR} & \multicolumn{3}{c}{HDR-NeRF-Syn -- LDR} \\
\cmidrule(lr){2-4} \cmidrule(lr){5-7}
& PSNR $\uparrow$ & SSIM $\uparrow$ & LPIPS $\downarrow$
& PSNR $\uparrow$ & SSIM $\uparrow$ & LPIPS $\downarrow$ \\
\midrule
GaussHDR$\dagger$
& \cellcolor{yellow!30}28.45 & \cellcolor{yellow!30}0.8970 & \cellcolor{yellow!30}0.087 & \cellcolor{yellow!30}12.56 & \cellcolor{yellow!30}0.6293 & \cellcolor{yellow!30}0.295 \\
GaussHDR
& 26.96 & 0.8828 & 0.114 & 11.99 & 0.5848 & 0.302 \\
\textbf{PriorGS (Ours)}
& \cellcolor{red!30}31.22 & \cellcolor{red!30}0.9569 & \cellcolor{red!30}0.038 & \cellcolor{red!30}22.96 & \cellcolor{red!30}0.8937 & \cellcolor{red!30}0.081 \\
\bottomrule
\end{tabular}
\end{center}
\end{table}

\subsection{Further Qualitative Results} \label{subsec:apx_qual_results}
\textbf{Additional HDR-NVS Results.} Complementing Figure \ref{fig:hdr_qualitative}, Figure \ref{fig:apx_hdr_qualitative} presents additional qualitative HDR-NVS results across two scenes and three different input exposures. We observe that GaussHDR and SeHDR recover the HDR radiance inaccurately and inconsistently across different input exposures, displaying severe color and geometry artifacts particularly in under-exposed and over-exposed regions. In contrast, PriorGS behaves much more consistently across input exposure levels and substantially outperforms the baselines in faithful HDR radiance recovery, consistent with our observations in Figure \ref{fig:hdr_qualitative}. These results further demonstrate the effectiveness of LDR-to-HDR prior supervision within our closed-loop reconstruction framework.
\begin{figure}[h]
    \begin{center}
    \includegraphics[width=0.95\linewidth]{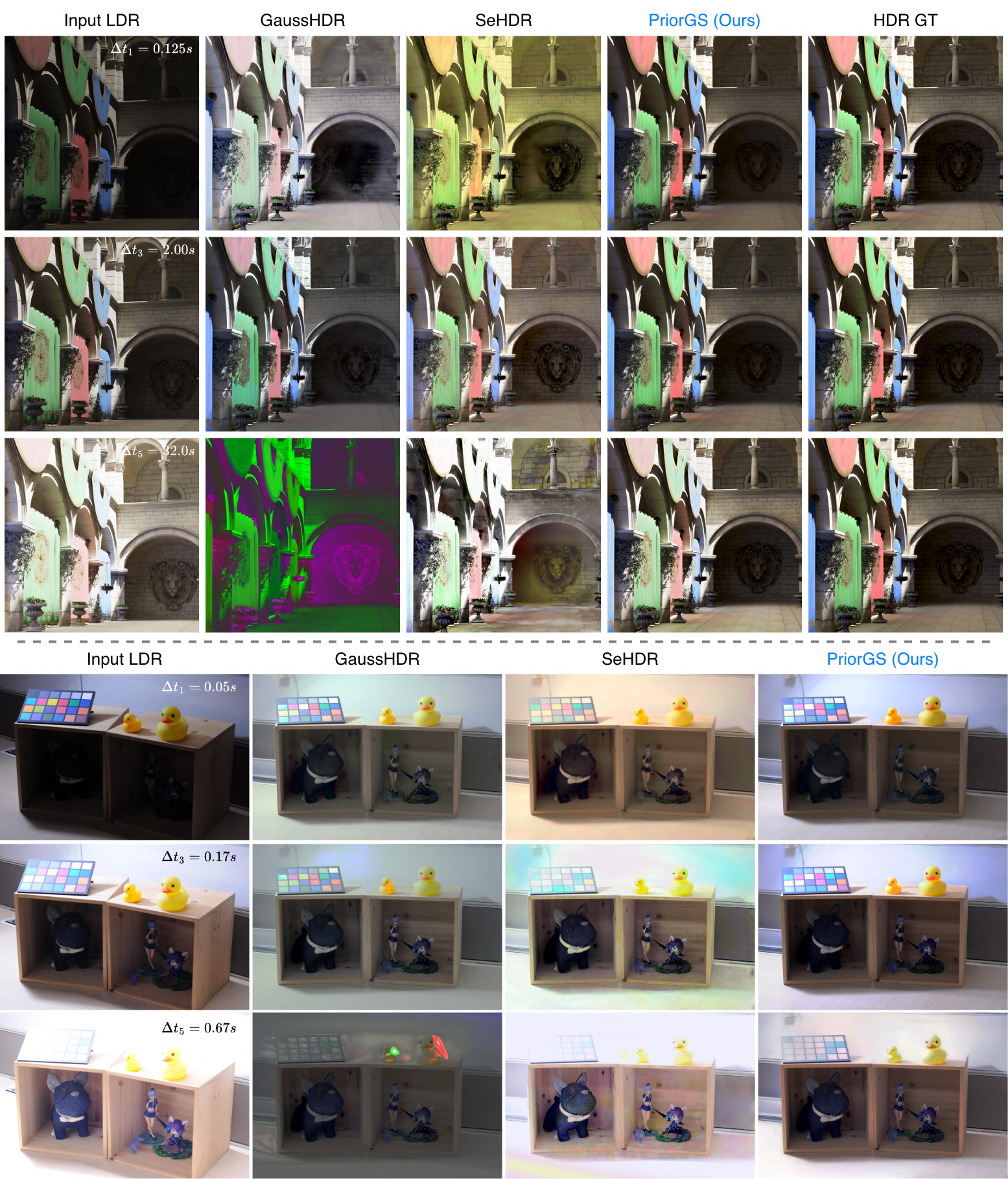}
    \end{center}
    \caption{Further qualitative HDR-NVS results given low- ($\Delta t_1$), medium- ($\Delta t_3$), and high-exposure ($\Delta t_5$) inputs across (\textit{top half}) synthetic and (\textit{bottom half}) real scenes.}
    \label{fig:apx_hdr_qualitative}
\end{figure}

\textbf{Additional Re-exposed LDR-NVS Results.} We complement Figure \ref{fig:ldr-qualitative} by presenting additional qualitative re-exposed LDR-NVS comparisons across two scenes in Figure \ref{fig:apx_ldr_qualitative}. Specifically, given single-exposure LDR inputs, we compare the methods’ novel-view LDR re-exposures at all five exposure levels, including the training (input) exposure. We observe that HDR-GS, GaussHDR, and Mono-HDR-GS reproduce novel-view LDR images at the observed training exposure with high fidelity, showing that their flexible MLP-based CRFs can closely fit the supervised exposure. However, these CRFs generalize poorly to unseen exposures, resulting in overexposed or underexposed renderings and color artifacts. This sharp gap between observed- and unseen-exposure performance indicates overfitting to the training exposure, consistent with our earlier discussions in Section \ref{sec:intro} and observations in Figure \ref{fig:ldr-qualitative}. SeHDR mitigates this overfitting and the resulting exposure errors and color artifacts by heuristically fixing a CRF for each scene. However, it still suffers from limited scene adaptability and exhibits a narrow response range due to inaccurate CRF preselection. In contrast, PriorGS produces substantially more accurate novel-view LDR re-exposures at unseen exposure levels while maintaining high fidelity at the observed training exposure, consistently outperforming all baselines. This behavior aligns with our earlier observations in Figure~\ref{fig:ldr-qualitative} and highlights the intended balance of our CRF formulation: the prior regularizes its behavior beyond the supervised exposure, while its learnable parameters exhibit excellent scene adaptability. Together, these results validate the need for a learnable CRF prior in single-exposure HDR-NVS and demonstrate the effectiveness of our proposed model.
\begin{figure}[h]
    \begin{center}
    \includegraphics[width=0.95\linewidth]{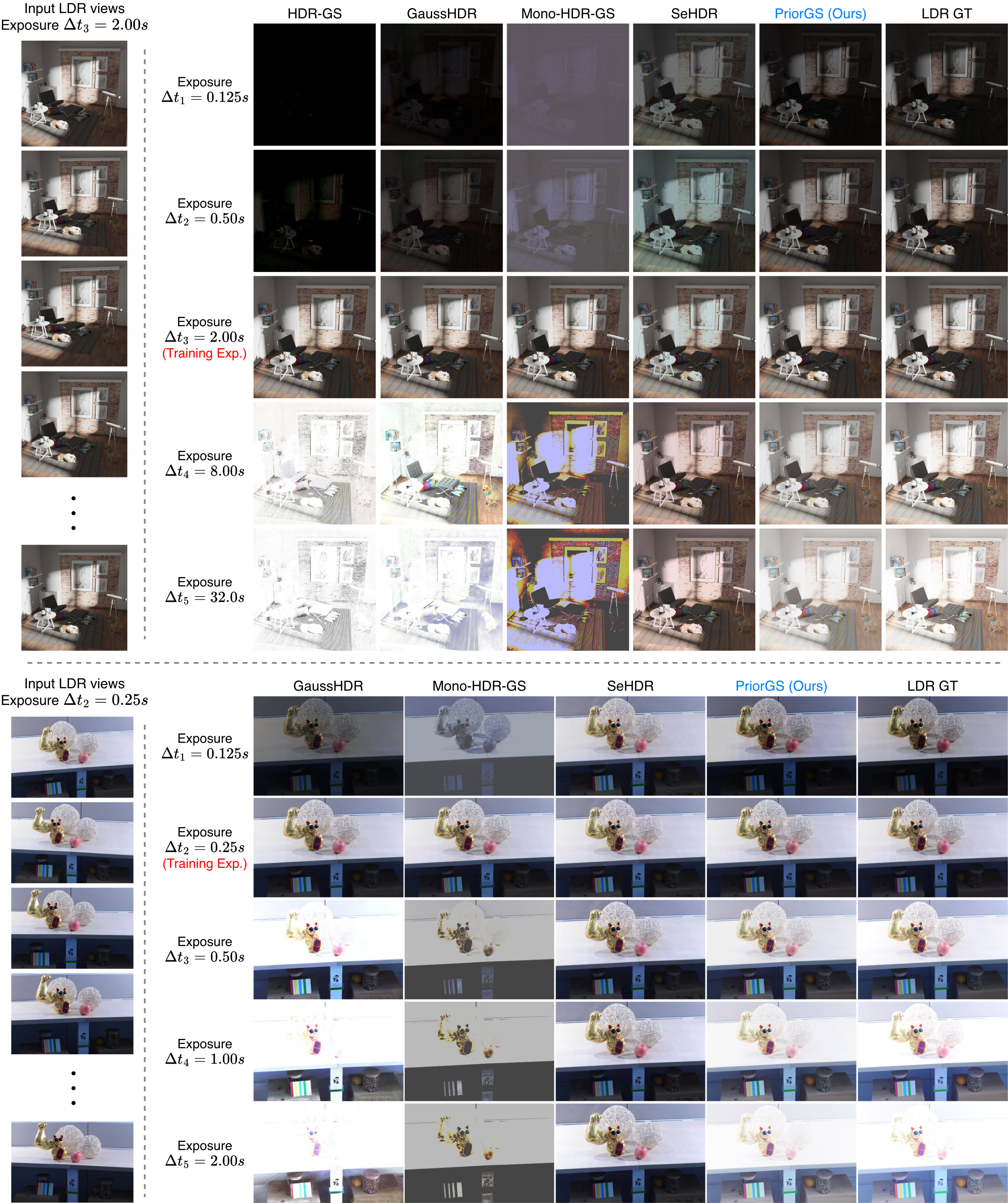}
    \end{center}
    \caption{Further qualitative re-exposed LDR-NVS results. Using single-exposure LDR views for training, we render a novel view at all five exposures, including the training exposure. The training exposure is (\textit{top half}) $\Delta t_3$ for a synthetic scene and (\textit{bottom half}) $\Delta t_2$ for a real scene.}
    \label{fig:apx_ldr_qualitative}
\end{figure}

\textbf{Qualitative Analysis of Learned CRF.} In Figure \ref{fig:crf-qual}, we compare our learned CRF curve with the MLP-based CRF of GaussHDR \citep{GaussHDR} for the synthetic \textit{bear} scene. Both methods are trained under the same single-exposure setting, using LDR views at exposure $\Delta t_1=0.125s$. We additionally present their novel-view LDR re-exposures at the training exposure $\Delta t_1=0.125s$, and the test exposure, $\Delta t_3=2.0s$. Both methods faithfully reproduce the novel-view LDR re-exposure at the training exposure, where the target intensities fall within the output range supervised by training observations. Beyond this range, however, the flexible MLP-based CRF curve exhibits pronounced divergence between its RGB responses, early saturation of the blue channel, and local violations of monotonicity. Without explicit priors to regularize the response inside this unsupervised region, its flexibility permits implausible response learning. Since rendering at the test exposure involves LDR intensities beyond the supervised region, MLP-based re-exposure leads to substantial brightness and color distortions with a blue tint. Our proposed CRF, however, leverages its DoRF-derived prior to maintain monotonicity and a more plausible response beyond the supervised range, producing substantially more faithful LDR renderings at the test exposure. We note that the differences between the learned CRFs within the supervised range is due to different HDR reconstructions by the two methods, as each CRF maps its respective reconstruction to the target exposures. We mark the approximate supervision boundary in Figure \ref{fig:crf-qual} at the $99.5$-th percentile of pixel intensities across all input LDR views. These results validate the need for CRF regularization with a learnable prior that ensures monotonicity and promotes plausible response learning.
\begin{figure}[t]
    \begin{center}
    \includegraphics[width=0.95\linewidth]{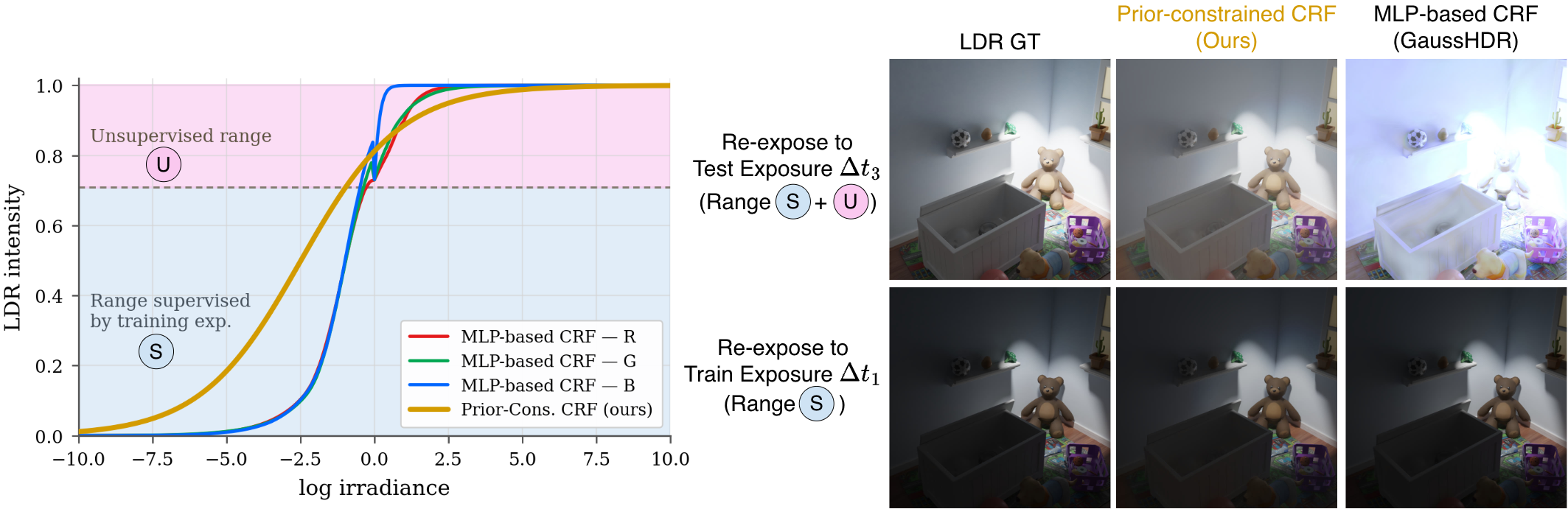}
    \end{center}
    \caption{Learned MLP-based and Prior-Constrained CRF curves and novel-view LDR re-exposures. Both CRFs faithfully reproduce the LDR view at training exposure, where target intensities fall within their CRFs' supervised range. Beyond this range, MLP-based CRF curve exhibits RGB divergence, early saturation, and nonmonotonic behavior, resulting in poor rendering at test exposure that involves mapping to CRF's unsupervised range. Our prior-constrained CRF maintains a plausible, monotonic response in this range, producing more faithful test-exposure renderings.}
    \label{fig:crf-qual}
\end{figure}

\subsection{Extension of 3D HDR Appearance Editing}
\label{subsec:apx_application}
\textbf{Discussion.} In Section \ref{subsec:perf_application}, we discussed that without a 3D gain field, our 3D HDR editing capability cannot be achieved by scaling Gaussians' HDR radiance, as this does not expand the dynamic range (DR) and is only equivalent to an exposure shift. Figure \ref{fig:hdr-hist} illustrates this point: Scaling 3D HDR radiance with $\rho$ only shifts the log-luminance histogram and leaves the DR unchanged. In contrast, increasing $\rho$ for 3D gain field scaling substantially expands the DR and changes the histogram. The histograms show the log-luminance distribution of HDR view 1 from the \textit{desk} scene shown in Figure \ref{fig:hdr-edit}, restricted to values between the $1$st and $99$th percentiles.
\begin{figure}[t]
    \centering

    \begin{subfigure}[b]{0.48\linewidth}
        \centering
        \includegraphics[width=\linewidth]{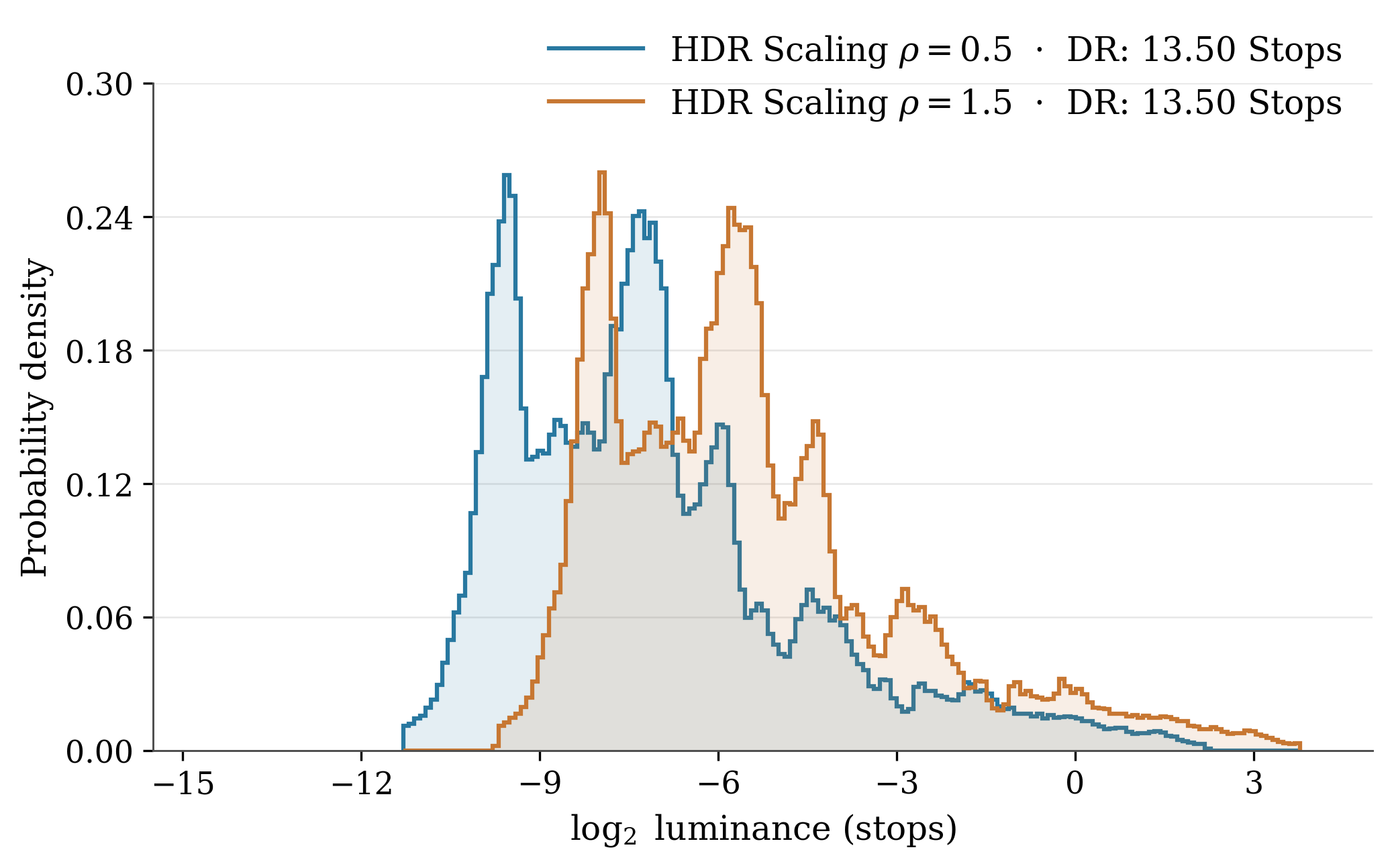}
        \caption{Log-luminance histogram; HDR scaling}
    \end{subfigure}
    \hfill
    \begin{subfigure}[b]{0.48\linewidth}
        \centering
        \includegraphics[width=\linewidth]{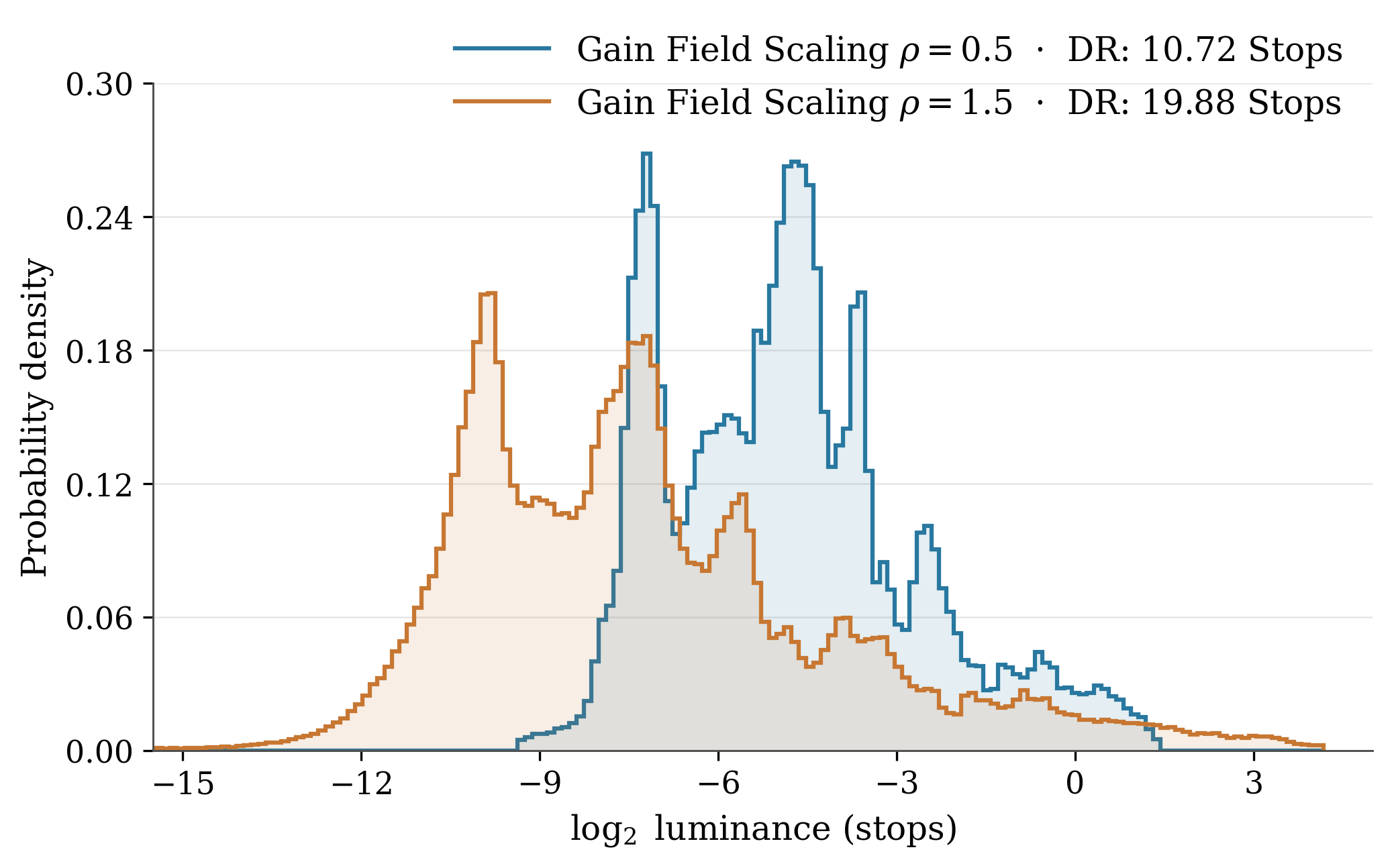}
        \caption{Log-luminance histogram; Gain Field scaling}
    \end{subfigure}

    \caption{Changes in an HDR view’s log-luminance histogram by varying $\rho$ from $0.5$ to $1.5$ when $\rho$ scales (a) 3D HDR radiance and (b) the 3D gain field. Increasing $\rho$ for 3D HDR radiance scaling only shifts the histogram (equivalent to an exposure shift) and leaves the dynamic range (DR) unchanged. For 3D gain-field scaling, it instead changes the histogram and substantially expands the DR.}
    \label{fig:hdr-hist}
\end{figure}

\textbf{Another Example of Global Editing.} Figure \ref{fig:hdr-edit-apx} complements Figure \ref{fig:hdr-edit} by demonstrating 3D HDR appearance editing on another scene using the same $\rho$-scaling of the 3D gain field. As observed before, increasing $\rho$ expands the DR and enhances the HDR appearance and chromatic contrast.
\begin{figure}[t]
    \begin{center}
    \includegraphics[width=0.95\linewidth]{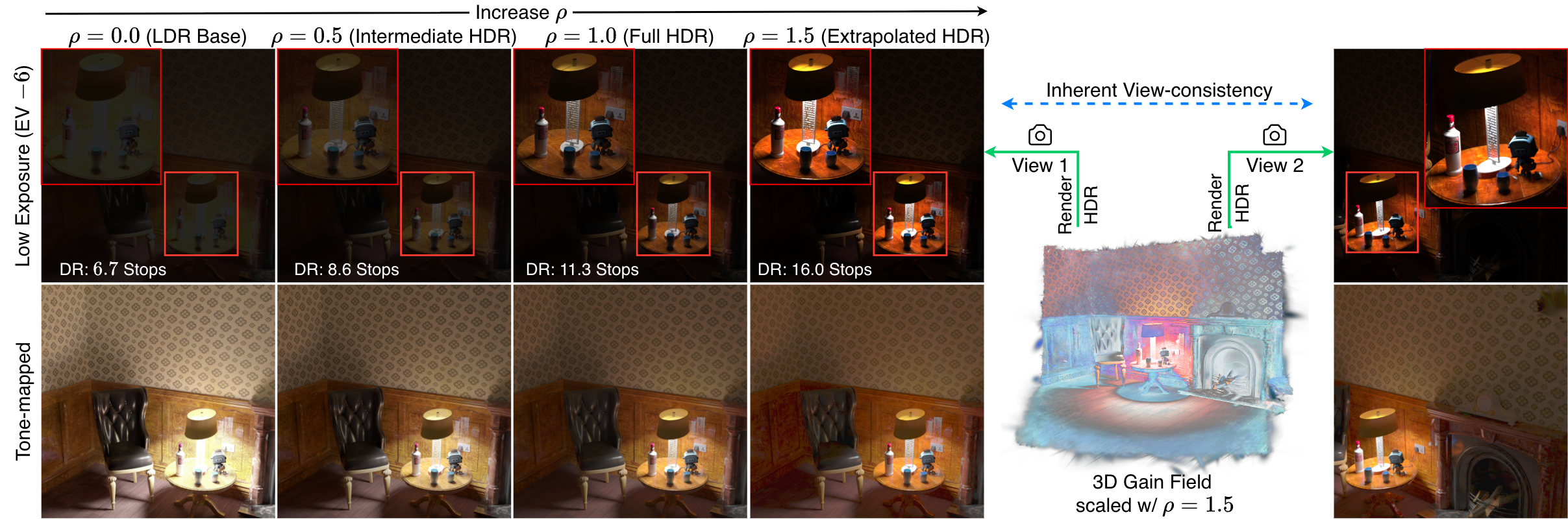}
    \end{center}
    \caption{Further example of 3D HDR appearance editing via gain field scaling with $\rho$. We sweep $\rho$ and show an HDR view at a low-exposure (\textit{top row}) and after a fixed tone-mapping (\textit{bottom row}). Increasing $\rho$ expands dynamic range (DR) and enhances HDR appearance and chromatic contrast. Since $\rho$ operates directly on 3D Gaussians, the edits are inherently view-consistent.}
    \label{fig:hdr-edit-apx}
\end{figure}

\textbf{Extension to Localized Editing.} We extend our 3D HDR appearance editing application to support localized edits. Specifically, instead of varying $\rho$ globally across all 3D Gaussians (i.e. scaling the gain field globally with $\rho$), we adjust it only for a manually selected subset of Gaussians, enabling targeted control over the HDR appearance of a specific scene region or object. Figure \ref{fig:hdr-local} presents the results. We demonstrate different local adjustments to $\rho$ across three scenes, each chosen to achieve a distinct HDR appearance in the desired region. For each scene, we show 2D visualizations of the 3D gain field before and after local editing. We also present HDR renderings from two viewpoints before and after the edits, displayed at a selected exposure. We observe that $\rho$ adjustments achieve the desired local appearance changes. Furthermore, as with global editing, the local changes are inherently view-consistent as $\rho$ adjustments are applied directly on 3D Gaussians.
\begin{figure}[h]
    \begin{center}
    \includegraphics[width=0.95\linewidth]{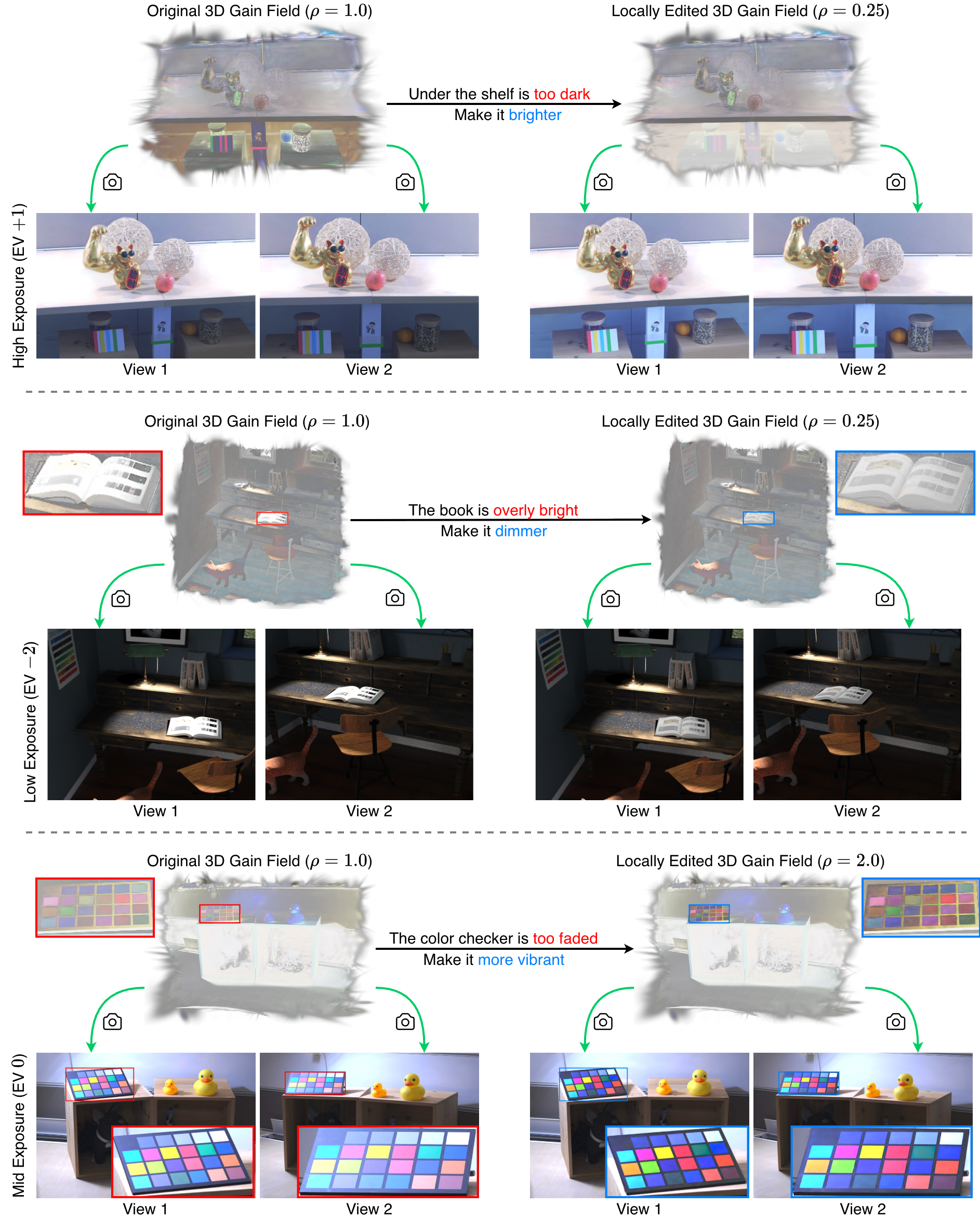}
    \end{center}
    \caption{Localized 3D HDR appearance editing by adjusting $\rho$ for a manually selected subset of 3D Gaussians across three scenes: (\textit{top}) brightening the region beneath a shelf, (\textit{middle}) dimming an overly bright book, and (\textit{bottom}) enhancing a color checker’s vividness. We show before-and-after 2D visualizations of the 3D gain field and HDR renderings from two viewpoints, displayed at a selected exposure. As $\rho$ operates directly in 3D, the edits are inherently view-consistent.}
    \label{fig:hdr-local}
\end{figure}

\subsection{Further Ablations} \label{subsec:apx_abl_study}
We provide qualitative evidence beyond those presented in Figure \ref{fig:main-abl} for the ablations reported in Table \ref{tab:ablation} of main paper. Furthermore, we carry out an ablation on LDR-base supervision.

\textbf{LDR-to-HDR Prior \& Closed-Loop Reconstruction.} Figure \ref{fig:apx_abl_hdr_rec} provides further qualitative ablations of our LDR-to-HDR reconstruction prior and closed-loop reconstruction.
Removing the LDR-to-HDR reconstruction prior yields significantly inaccurate HDR radiance recovery, while disabling the reconstruction loop introduces artifacts and degrades HDR quality. This further confirms the need for learned LDR-to-HDR reconstruction priors for high-fidelity HDR-NVS from single-exposure views and demonstrates the HDR quality gains from closed-loop reconstruction.
\begin{figure}[h]
    \begin{center}
    \includegraphics[width=0.95\linewidth]{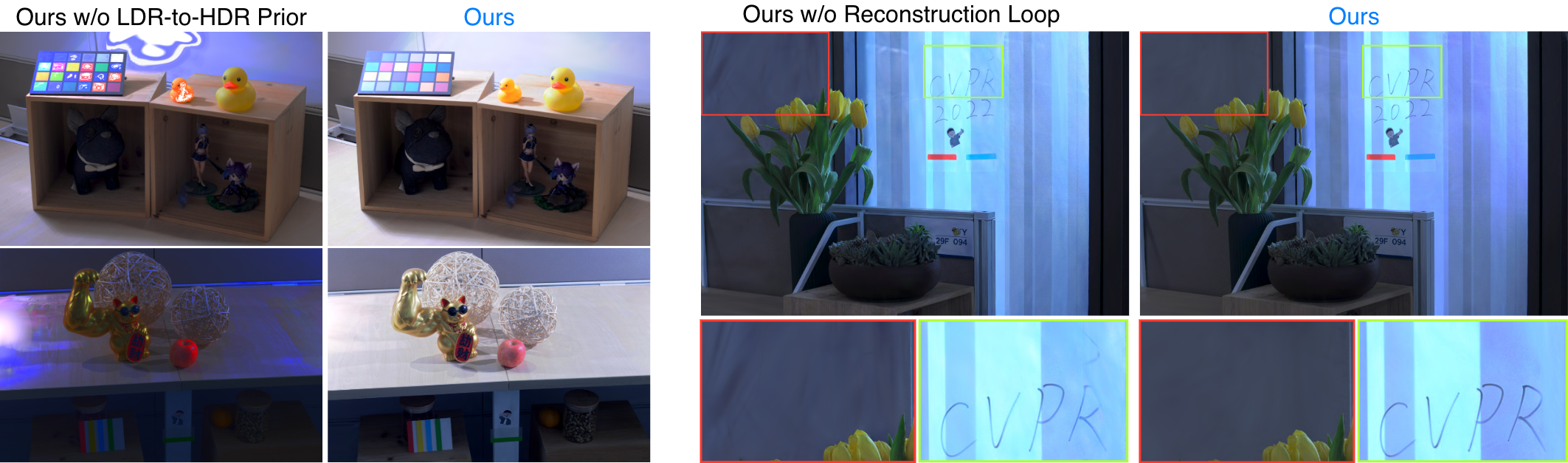}
    \end{center}
    \caption{Further qualitative ablations on HDR views for (\textit{left}) removing LDR-to-HDR reconstruction priors and (\textit{right}) disabling the reconstruction loop. Removing the LDR-to-HDR reconstruction prior results in inaccurate HDR radiance recovery, while disabling the reconstruction loop degrades HDR reconstruction quality. All HDR views are tone-mapped identically for visualization.}
    \label{fig:apx_abl_hdr_rec}
\end{figure}

\textbf{Base-Gain HDR Representation.} Figure \ref{fig:apx_abl_direct_hdr} presents additional qualitative ablations, where we replace our base-gain HDR representation with direct HDR encoding. We note that in our implementation of direct HDR encoding, each Gaussian’s color SH coefficients encode log-radiance. This log-space formulation follows prior works \citep{HDRGS, GaussHDR} and, in our experiments, improves training stability and HDR reconstruction quality compared with linear encoding. We also increase the SH degree of the direct HDR encoding variant to match the color representation capacity of our base–gain representation and ensure fair comparison. 
As shown in Figure \ref{fig:apx_abl_direct_hdr}, cross-view inconsistencies among independently generated pseudo-HDR targets cause severe color and geometry artifacts. However, our base-gain HDR representation promotes robustness to those inconsistencies and achieves superior performance in HDR reconstruction.
Additionally, in Figure \ref{fig:apx_abl_gain_to_mlp}, we present further qualitative ablations for replacing the gain field with MLPs that predict each Gaussian's HDR radiance from its LDR-base color $\vc$, position $\vmu$, and viewpoint direction $\vd$ ($\operatorname{MLP}(\vc,\vmu,\vd)$) or base color $\vc$ alone ($\operatorname{MLP}(\vc)$). 
Both MLP variants exhibit poor performance in reconstructing high-frequency textures in challenging dark regions. For the variant conditioned on LDR-base color, position, and viewing direction, this behavior is consistent with the low-frequency bias of MLPs operating on raw positions and view directions, which hinders reconstruction of sharp surface boundaries and detailed textures \citep{Nerf}. The color-only variant, on the other hand, predicts identical HDR radiance for identical base colors. However, Gaussians with the same base color may have different HDR radiance values, and this variant lacks local texture and broader context to distinguish these cases. Our base-gain representation, however, overcomes all discussed limitations and recovers high-frequency textures with higher fidelity. 
Finally, we note that direct HDR encoding and MLP variants preclude 3D HDR editing application enabled by the gain field, which further highlights the benefits of our base-gain representation. 
\begin{figure}[t]
    \begin{center}
    \includegraphics[width=0.95\linewidth]{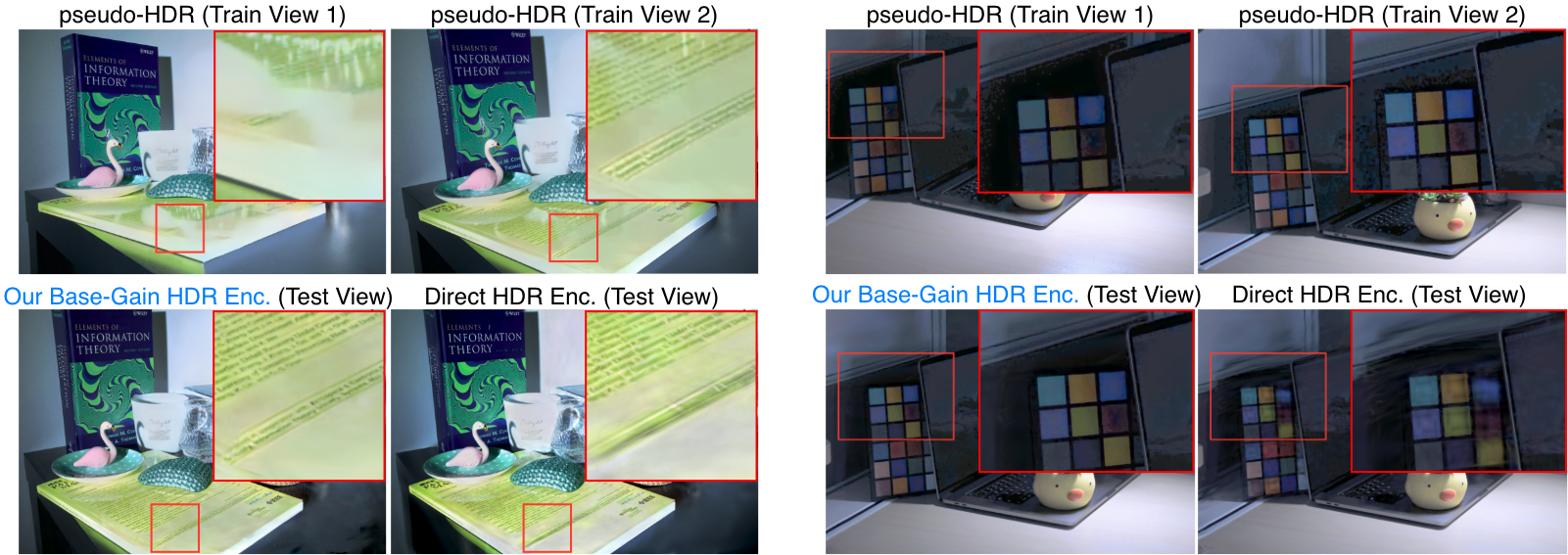}
    \end{center}
    \caption{Further qualitative comparison (ablation) of our base-gain HDR representation against direct HDR encoding. (\textit{top row}) Cross-view inconsistencies in pseudo-HDR targets (\textit{bottom row}) produce severe artifacts with direct HDR encoding, while our base–gain representation promotes robustness to such inconsistencies and achieves higher quality HDR reconstruction.}
    \label{fig:apx_abl_direct_hdr}
\end{figure}
\begin{figure}[h]
    \begin{center}
    \includegraphics[width=0.95\linewidth]{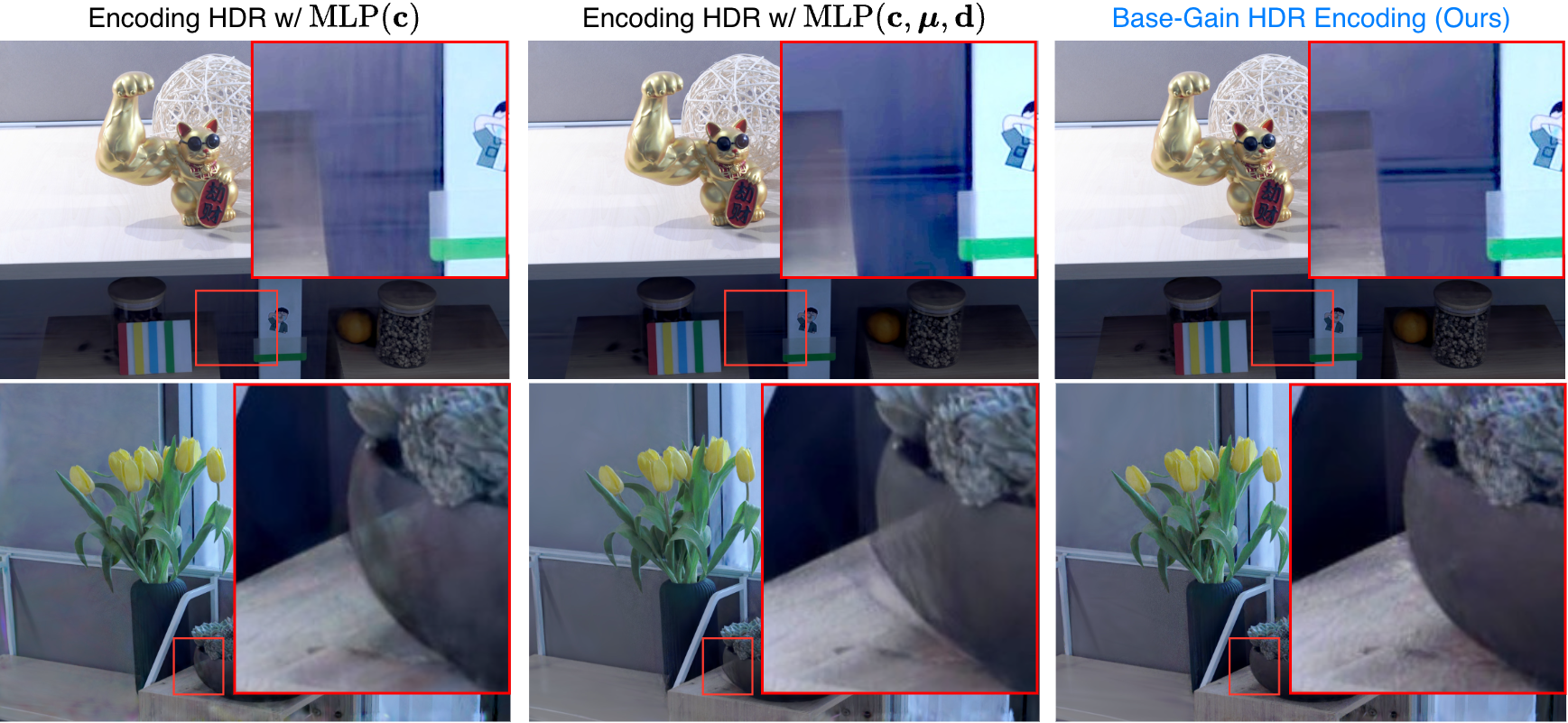}
    \end{center}
    \caption{Further qualitative ablations on HDR views for replacing our gain field with two MLP variants: one predicts each Gaussian's HDR radiance from its LDR-base color $\vc$ alone; the other also uses position $\vmu$ and direction $\vd$. Our base-gain representation outperforms both MLP variants in reconstructing high-frequency details in regions with challenging lighting.}
    \label{fig:apx_abl_gain_to_mlp}
\end{figure}

\textbf{Prior-Constrained CRF Modeling.} In Figure \ref{fig:apx_abl_crf}, we ablate our proposed CRF with (1) fixed gamma encoding $\gamma = 2.2$; (2) an MLP applied on 2D log-irradiance obtained by exposing HDR renderings; and (3) an MLP applied on 3D log-irradiance obtained by exposing per-Gaussian HDR radiance. Following the recommendation of \citet{HDRGS} and \citet{GaussHDR}, we feed MLPs log-transformed irradiance rather than linear irradiance to improve training stability. The MLP-based CRFs closely fit the training exposure but generalize poorly to the unseen exposure, producing overexposure and severe color artifacts. Fixed gamma encoding also fails to reproduce the unseen exposure as it cannot adapt to individual scenes. Our CRF, however, achieves substantially higher fidelity than all alternatives at the unseen exposure: its strong prior regularizes its behavior beyond the observed exposure, while its learnable parameters enable scene-specific adaptation.
\begin{figure}[h]
    \begin{center}
    \includegraphics[width=0.95\linewidth]{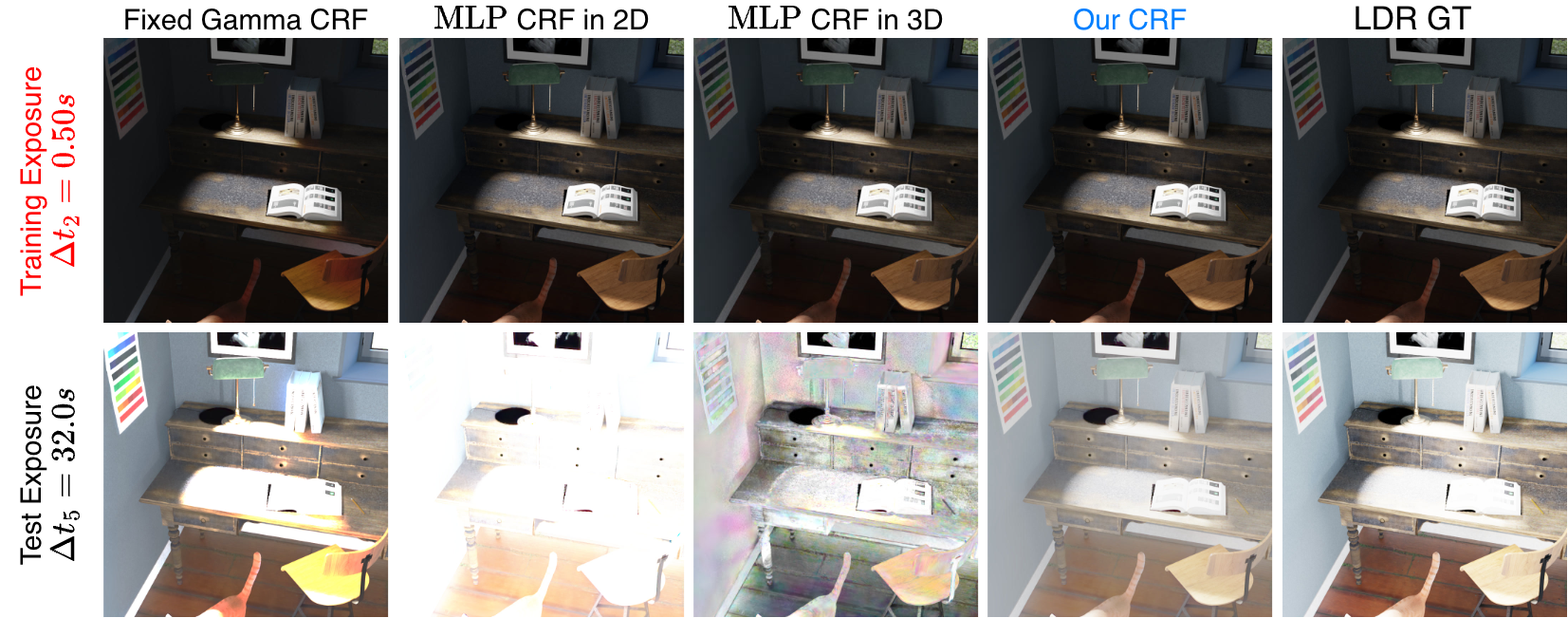}
    \end{center}
    \caption{Qualitative ablations on LDR views for replacing our proposed CRF with MLPs in 2D and 3D and with fixed gamma CRF. All configurations reproduce the training exposure, yet our CRF achieves much higher reconstruction fidelity at the test exposure.}
    \label{fig:apx_abl_crf}
\end{figure}

\textbf{LDR-Base Supervision.} In Figure \ref{fig:apx_ldr_base}, we qualitatively ablate LDR-base supervision by disabling its loss term in Equation \ref{eq:loss}. We observe that absence of LDR-base supervision severely impairs scene reconstruction, causing failure in texture recovery around the scene. This highlights the necessity of LDR-base supervision for robust HDR scene reconstruction. We further note that without LDR-base supervision HDR radiance is encoded arbitrarily across the LDR base and gain field, which lose their intended roles as LDR appearance and its HDR extension. This loss of a meaningful base–gain decomposition further precludes the 3D HDR editing application enabled by our base-gain representation and forgos joint reconstruction of the 3D LDR-base and HDR scenes.
\begin{figure}[h]
    \begin{center}
    \includegraphics[width=0.89\linewidth]{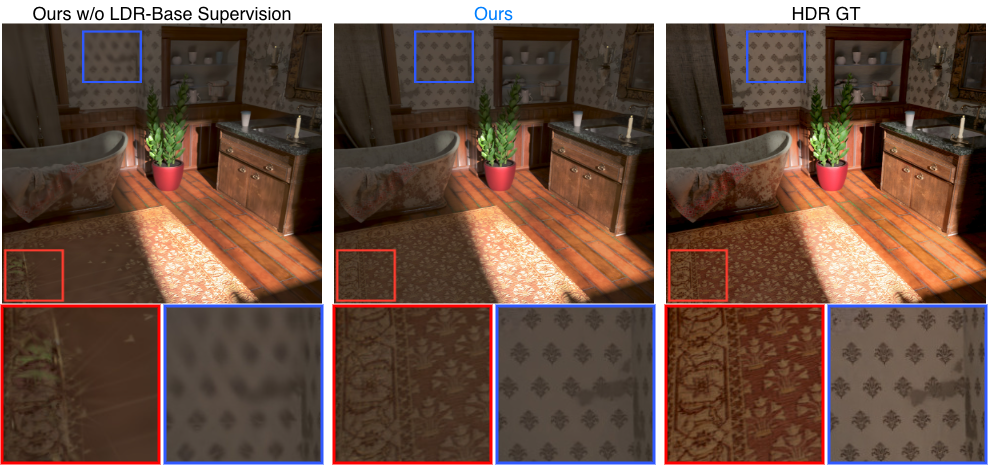}
    \end{center}
    \caption{Qualitative ablation of LDR-base supervision on HDR views. Removing this supervision severely impairs scene reconstruction, causing failure in texture recovery all across the scene.}
    \label{fig:apx_ldr_base}
\end{figure}

\end{document}

%% file: math_commands.tex
\usepackage{amsmath,amsfonts,bm}

\def\eqref#1{equation~\ref{#1}}

\def\1{\bm{1}}

\def\vmu{{\bm{\mu}}}
\def\vtheta{{\bm{\theta}}}

\def\vc{{\bm{c}}}
\def\vd{{\bm{d}}}

\def\vg{{\bm{g}}}
\def\vh{{\bm{h}}}

\def\vl{{\bm{l}}}

\def\mR{{\bm{R}}}
\def\mS{{\bm{S}}}

\def\mSigma{{\bm{\Sigma}}}

\DeclareMathAlphabet{\mathsfit}{\encodingdefault}{\sfdefault}{m}{sl}
\SetMathAlphabet{\mathsfit}{bold}{\encodingdefault}{\sfdefault}{bx}{n}

\def\sG{{\mathbb{G}}}

\newcommand{\R}{\mathbb{R}}

